\documentclass[a4paper,fleqn]{cas-sc}

\usepackage[authoryear,longnamesfirst]{natbib}

\usepackage{graphicx}%
\usepackage{multirow}%
\usepackage{amsmath,amssymb,amsfonts}%
\usepackage{amsthm}%
\usepackage{mathrsfs}%
\usepackage[title]{appendix}%
\usepackage{xcolor}%
\usepackage{textcomp}%
\usepackage{manyfoot}%
\usepackage{booktabs}%
\usepackage{algorithm}%
\usepackage{algorithmicx}%
\usepackage{algpseudocode}%
\usepackage{listings}%

\usepackage{tikz}
\usetikzlibrary{positioning,arrows.meta,shapes.geometric}
\usepackage{hyperref}
\usepackage{cleveref}
\usepackage{longtable}
\usepackage{booktabs}
\usepackage{array}
\usepackage{tabularx}
\usepackage[table]{xcolor}
\usepackage{subcaption}
\usepackage{siunitx}
\usepackage{makecell}
\usepackage{todonotes}

\theoremstyle{thmstyleone}%
\newtheorem{theorem}{Theorem}
\newtheorem{proposition}[theorem]{Proposition}%

\theoremstyle{thmstyletwo}%
\newtheorem{remark}{Remark}%
\newtheorem{corollary}{Corollary}%

\theoremstyle{thmstylethree}%
\newtheorem{definition}{Definition}%

\def\tsc#1{\csdef{#1}{\textsc{\lowercase{#1}}\xspace}}
\tsc{WGM}
\tsc{QE}

\begin{document}
\let\WriteBookmarks\relax
\def\floatpagepagefraction{1}
\def\textpagefraction{.001}

\shorttitle{Physics-Conforming Latent Twins}

\shortauthors{Chung et al.}

\title [mode = title]{Physics-Conforming Latent Twins}



%

\author[1]{Matthias Chung}[orcid=0000-0001-7822-4539,]

\cormark[1]


\ead{matthias.chung@emory.edu}

\ead[url]{}

\credit{Conceptualization, Methodology, Formal analysis, Investigation, Software, Validation, Writing – original draft, Writing – review and editing, Supervision, Project administration}

\author[1]{Yutong Bu}[orcid=0009-0008-5511-7341]

\ead{audrey.bu@emory.edu}

\ead[url]{}

\credit{Methodology, Software, Validation, Investigation, Visualization, Writing – original draft, Writing – review and editing}

\affiliation[1]{organization={Department of Mathematics, Emory University},
            addressline={400 Dowman Drive},
            city={Atlanta},
            postcode={30322},
            state={GA},
            country={USA}}




\author[2]{Deepanshu Verma}[orcid=0000-0002-0904-5470]


\ead{dverma@clemson.edu}

\ead[url]{}

\credit{Methodology, Formal analysis, Software, Validation, Investigation, Writing – original draft, Writing – review and editing}

\affiliation[2]{organization={School of Mathematical and Statistical Sciences, Clemson University},
            addressline={220 Parkway Drive},
            city={Clemson},
            postcode={29634},
            state={SC},
            country={USA}}

\cortext[cor1]{Corresponding author}

\begin{abstract}
Surrogate models are central to scientific machine learning, where they enable fast prediction, simulation, inference, and control for complex physical systems. For time-dependent problems, however, accurate interpolation of training trajectories is not sufficient: reliable surrogates should also respect the conservation laws, invariants, admissibility conditions, and dissipative structures that give those trajectories physical meaning. We introduce Physics-conforming Latent Twins, a framework for learning latent surrogate solution operators whose dynamics satisfy selected physical principles by design. The method builds on the Latent Twin formulation by jointly learning an encoder, a decoder, and a latent flow map between arbitrary time-indexed states, while constraining the latent dynamics to preserve or dissipate prescribed structural quantities. We develop a constraint-transfer viewpoint that separates pullback compatibility from latent conformity, connecting physical structure in the original state space with enforceable constraints in latent space, and prove structure-preservation bounds showing how latent enforcement improves control of physical defects after decoding. We also derive algebraic conditions for latent flow maps that preserve linear and quadratic invariants or enforce dissipative inequalities. Numerical experiments on representative ODE and PDE benchmarks demonstrate improved constraint satisfaction, structural fidelity, and qualitative long-time behavior while maintaining accurate surrogate prediction.
\end{abstract}



\maketitle

\section{Introduction}\label{sec:introduction}

The promise of scientific machine learning is not merely to accelerate simulation, but to learn models that retain the organizing structure of the physical world. Time-dependent physical systems are governed by more than input-output correlations: they obey conservation laws, preserve invariants, remain within admissible states, and often reflect underlying symmetries or geometric structure. When learned surrogates ignore these principles, they may interpolate data well while losing the properties that make scientific predictions reliable under extrapolation, long-time simulations, and other downstream use. This motivates a central question: how can we learn surrogate solution operators whose dynamics are accurate, efficient, and physically meaningful by design?

In this work, we emphasize a structure-aware viewpoint on scientific machine learning: conservation, symmetry, invariance, admissibility, and dissipation are not merely consequences of a particular system of governing equations; they are underlying structural principles that delimit the space of physically admissible evolutions. From this perspective, governing equations are expressions of deeper organizing principles, rather than the sole source of physical structure. We therefore seek surrogate models that do more than imitate simulated trajectories: their learned dynamics are restricted by the principles that give those trajectories physical meaning.

This viewpoint is aligned with a recurring theme in applied mathematics and physics: an appropriate representation can reveal structure that is hidden in the original variables. Noether's theorem is a classical example, identifying conservation laws as consequences of continuous symmetries of the action rather than as isolated properties of particular equations \cite{noether1918invarianten,noether1983invariante}. A surrogate model that encodes this structural viewpoint directly, constraining its latent dynamics to respect the organizing principles of the physical system, may be called physics-conforming. \emph{Physics-conforming Latent Twins} pursue this goal by using the latent space not merely for compression, but as the space in which selected physical structures are encoded and propagated by design.

The \emph{Latent Twins} framework was introduced in \cite{chung2026latent} and jointly learns (i) an encoder--decoder mapping high-dimensional states to a latent representation, and (ii) a latent solution operator that propagates states between arbitrary times $s,t$. Building upon this foundation, we introduce structured latent flow maps that satisfy selected physical principles, including conservation laws, quadratic invariants, and dissipative inequalities, thereby enabling physics-conforming surrogate dynamics. This places our approach within the operator learning viewpoint: rather than learning one-step updates or vector fields, we learn mappings of the form $u(s) \mapsto u(t)$. The latent space is learned jointly with the dynamics so that the resulting representation admits surrogate evolution satisfying the prescribed physical laws.

Among the many surrogate modeling approaches developed in scientific machine learning, reduced-order models, Koopman methods, neural operators, physics-informed networks, and latent representation learning, reviewed in \Cref{sec:surrogate-positioning}, approaches differ in the mathematical object they learn and in how physical structure is imposed. Within this landscape, Physics-conforming Latent Twins are positioned as latent solution-operator models: the central object is a learned latent flow map between time-indexed states, rather than a projected differential equation, an identified vector field, or a residual-constrained solution field. Rather than treating the latent space as a passive compression of snapshots, we use the learned representation as the space in which the surrogate evolution is organized and constrained — physical structure is not an auxiliary diagnostic imposed after prediction, but part of what the learned solution operator is designed to encode and preserve.

The main contributions of this work are to develop, analyze, and demonstrate a framework for Physics-conforming Latent Twins. We formulate Latent Twins as latent surrogate solution operators whose latent flow maps can be constrained directly, rather than relying only on penalties or diagnostics in the physical variables. To connect physical constraints with latent dynamics, we introduce a constraint-transfer viewpoint that relates structure in the state space to compatible structure in the latent space. This viewpoint separates two mechanisms that are often conflated: identifying a latent representation of a physical quantity, and enforcing its preservation or dissipation by the latent dynamics. This leads to structural approximation bounds showing that a posteriori violations of prescribed physical constraints are controlled by the Latent Twins approximation error, with sharper estimates when compatible constraints are encoded and enforced directly in latent space. We then specialize the framework to linear latent flow maps, deriving algebraic conditions for preserving linear invariants, quadratic energy structures, and dissipative inequalities. Finally, we demonstrate the framework through ODE and PDE experiments that test reconstruction, latent evolution, pullback compatibility, constraint preservation, and long-time behavior, including stable linear systems, nonlinear dynamics with conservation or energy structure, dissipative gradient flows, and PDE benchmarks with conservative and dissipative structure.

The remainder of the paper is organized as follows. In \Cref{sec:background}, we position Physics-conforming Latent Twins relative to related surrogate modeling approaches, formalize the Latent Twin solution-operator framework, and introduce structural functionals for conservation, admissibility, and dissipation; these topics are developed in \Cref{sec:surrogate-positioning,sec:latent-twins,sec:physicsconf}, respectively. In \Cref{sec:method}, we derive approximation and structure-preservation results, including a posteriori control of structural defects in \Cref{sec:posteriorConform} and a priori bounds in \Cref{sec:aprioriConform} that sharpen when constraints are enforced in latent space. We also introduce pullbacks and physics-conforming latent flows and discuss algebraic realizations for linear invariants, quadratic invariants, and dissipative structures in \Cref{sec:pclflow}. Numerical experiments in \Cref{sec:numerics} demonstrate the proposed framework on ODE benchmarks and PDE examples, discussed in \Cref{sec:ODEs,sec:PDEs}. We conclude in \Cref{sec:conclusion} by discussing limitations and future directions.

\section{Background}\label{sec:background}

This section establishes the background needed to formulate Physics-conforming Latent Twins as structure-aware surrogate solution operators. We first situate the framework among related approaches in scientific machine learning, emphasizing the distinction between learning projected equations, residual-constrained solutions, vector fields, and latent solution maps. We then recall the Latent Twin formulation and fix notation for the encoder, decoder, and latent flow map. Finally, we introduce structural functionals to represent conservation, admissibility, and dissipation within a common mathematical framework. These ingredients set up the central question addressed in the remainder of the paper: how physical structure in the state space can be represented, enforced, and controlled through constraints on the latent dynamics.

\subsection{Positioning among surrogate modeling approaches} \label{sec:surrogate-positioning}
Scientific machine learning has produced a broad range of surrogate modeling approaches for time-dependent physical systems, and a comprehensive review is beyond the scope of this work. Rather than attempting an exhaustive survey, we focus on methodological directions most closely connected to the Latent Twin framework: reduced-order modeling, Koopman-based methods, neural operators, physics-informed neural networks, and latent representation learning for scientific dynamics. These approaches share the goal of replacing expensive time-dependent simulations by efficient surrogate models, but they differ in the mathematical object they choose to learn.

Classical reduced-order modeling starts from the observation that many scientific systems evolve on or near low-dimensional solution manifolds. Projection-based methods such as proper orthogonal decomposition and reduced-basis methods therefore construct a low-dimensional trial space and approximate the governing equations within that space \cite{benner2015survey,hesthaven2016certified}. This perspective has been foundational for scientific computing: it shows that efficient surrogates are possible when the dominant dynamical degrees of freedom can be identified.

A second line of work is based on the Koopman viewpoint. Instead of reducing the state directly, Koopman methods lift the dynamics to a space of observables in which nonlinear evolution can be represented linearly, at least formally \cite{koopman1931hamiltonian,mezic2005spectral}. Dynamic mode decomposition and its variants provide practical finite-dimensional approximations of this idea from data and may even add conservation laws, \cite{schmid2010dynamic,kutz2016dynamic,baddoo2023physics}. More recently, Koopman autoencoders use neural networks to learn coordinates in which the latent evolution is approximately linear \cite{lusch2018deep,takeishi2017learning}. These methods have moved the field forward by connecting representation learning with interpretable linear latent dynamics, but they are typically organized around evolution in learned coordinates rather than around physics-conforming solution operators.

A third related direction is sparse model discovery and system identification, most prominently the sparse identification of nonlinear dynamics (SINDy) framework \cite{brunton2016discovering}. These methods seek explicit governing equations by selecting a small number of active terms from a prescribed library of candidate functions, with extensions to PDEs, weak formulations, implicit dynamics, control inputs, and constrained or physics-informed variants \cite{rudy2017data,schaeffer2017learning,kaheman2020sindy,reinbold2020using}. Physics-conforming Latent Twins share the emphasis on interpretable and structured dynamics, but differ in the mathematical object being learned: rather than identifying an explicit vector field or differential equation, we learn a latent solution operator whose admissible evolution is constrained by structural functionals.

A fourth major direction is physics-informed learning. Physics-informed neural networks incorporate known differential equations directly into the training objective, typically through residual losses evaluated at collocation points \cite{raissi2019physics,baez2024guaranteeing, patel2022thermodynamically}. These approaches have reshaped the way data and governing equations are combined, especially in settings where observations are sparse or boundary and initial conditions are known. At the same time, the physical structure is usually imposed as a residual penalty in the ambient variables, rather than through a structured latent evolution operator.

A related line of work embeds physical structure directly into the neural network architecture rather than as a penalty. Hamiltonian Neural Networks \cite{greydanus2019hamiltonian} parametrize the vector field by learning a scalar Hamiltonian and deriving dynamics through Hamilton's equations, thereby guaranteeing exact energy conservation by construction. This inductive bias leads to faster training and better generalization on conservative systems such as the two-body problem and the pendulum, and produces models that are perfectly time-reversible. While Hamiltonian Neural Networks operate at the level of the vector field in the original state space, Physics-conforming Latent Twins pursue a complementary goal: enforcing analogous conservative structure on the latent solution operator rather than on the learned dynamics directly.

Neural operators take another step by learning mappings between function spaces. Architectures such as DeepONet and Fourier neural operators learn input-output maps associated with parametric differential equations or time-dependent solution maps \cite{lu2021learning,li2020fourier,kovachki2023neural}. This operator-learning viewpoint has been highly influential because it shifts attention from learning individual solutions to learning solution maps that generalize across families of inputs. Our work adopts this solution-map perspective, but asks how such maps can be represented in latent coordinates where physical structure can be imposed directly on the learned evolution.

Finally, recent representation-learning approaches combine latent variables with learned dynamics, operator maps, or coordinate-based functional representations. These methods ask whether high-dimensional scientific states can be represented in compact coordinates where prediction becomes easier. Latent space dynamics identification methods, such as LaSDI and its variants, use autoencoders to compress PDE solution data and then learn reduced dynamics in the latent space \cite{fries2022lasdi,he2023glasdi,tran2024weak,park2024tlasdi}. Related neural-field approaches use coordinate-based decoders or implicit neural representations to represent states as continuous functions, enabling prediction across irregular grids, complex geometries, and varying resolutions \cite{chen2022crom,serrano2023operator,yin2022continuous}. These works have substantially advanced latent and functional representations for scientific dynamics. Latent Twins share this broad goal, but differ in their organizing principle: the central object is a latent solution operator, learned jointly with the encoder and decoder, and designed to support physics-conforming structure between arbitrary states and times.

\subsection{Latent Twins}
\label{sec:latent-twins}

We now recall the solution-operator formulation of Latent Twins and introduce the notation used throughout the analysis.

\begin{definition}[Latent Twin]\label{def:latenttwin}
Let $\Phi$ denote the solution operator of a time-evolutionary system on a normed vector space $\mathcal U$,
\[
    \Phi(t,s,\cdot):\mathcal U\to \mathcal U,
\]
where $\Phi(t,s,u_s)$ maps a state $u_s\in\mathcal U$ at time $s$ to the corresponding state at time $t$. A \emph{Latent Twin} of $\Phi$ is a surrogate solution operator of the form
\begin{equation}\label{eq:latenttwinfinite}
    \Psi(t,s,u_s) := d\bigl(m_{s\to t}(e(u_s))\bigr),
\end{equation}
where $Z$ is a latent space, $e:\mathcal U\to Z$ is an encoder, $m_{s\to t}:Z\to Z$ is a latent flow map, and $d:Z\to\mathcal U$ is a decoder. The maps $e$, $m_{s\to t}$, and $d$ are trained so that
$$
    d\circ e \approx \mathrm{id}_{\mathcal U}
$$
on the relevant subset of $\mathcal U$, and
$$
    \Psi(t,s,u_s)\approx \Phi(t,s,u_s)
$$
for the time pairs and states of interest.
\end{definition}

\begin{remark}[Notation, time dependence, and state spaces]
In practice, the encoder, latent flow map, and decoder depend on trainable parameters $\theta$,
$$
e_\theta:\mathcal U\to Z,\qquad
m_{\theta,s\to t}:Z\to Z,\qquad
d_\theta:Z\to\mathcal U.
$$
For readability, we suppress the dependence on $\theta$ whenever no ambiguity arises. The notation is intentionally abstract in both time and state space. The indices $s$ and $t$ may denote continuous times, discrete time levels, or more general ordered parameters indexing system states. Thus, $m_{s\to t}$ should be understood as a latent map between two indexed states, and no temporal ordering (e.g., $s<t$) is imposed at the level of notation. For irreversible or strongly directional systems, however, approximating inverse-time maps may be ill-posed or unstable.

The state space $\mathcal U$ is likewise kept abstract. It may be finite-dimensional, as in a fully discrete dynamical system or ordinary differential equation, or infinite-dimensional, as in a PDE posed on a function space. In the latter case, practical encoder and decoder architectures are typically implemented after discretization or observation; the approximation results below make this step explicit through projection and reconstruction maps. Alternatively, one may regard the encoder and decoder themselves as operators between function spaces, as in operator-learning or neural-field formulations \cite{lu2021learning,li2020fourier,sitzmann2020implicit}.

Finally, although the latent space is often chosen to be lower-dimensional than the state space, $\dim(Z)\ll \dim(\mathcal U)$ is not required. Larger latent spaces may be useful for auxiliary variables, structural constraints, or richer latent dynamics \cite{chung2025sparse}.
\end{remark}

Having fixed the solution-operator notation, we now describe the structural properties that we aim to transfer from the physical dynamics to the latent flow map.

\subsection{Physics-conforming properties and approximation setting} \label{sec:physicsconf}
Time-evolutionary systems often preserve, restrict, or dissipate quantities of physical relevance. Examples include conserved mass or energy, dissipative inequalities, admissibility conditions such as positivity or boundedness, and symmetry constraints \cite{murray2007mathematical,Khalil2002,HairerWanner1996}. These properties are not merely side conditions to be checked after prediction; they are organizing principles that shape the class of admissible dynamics. For instance, in the SIR model, the relation $S+I+R=\mathrm{const.}$ identifies an invariant manifold on which the physically meaningful dynamics evolve \cite{murray2007mathematical}. Conversely, the assumption of a closed population is itself a structural modeling principle: it restricts the admissible class of epidemic models to dynamics that preserve the total population, of which the classical SIR equations are one example.

To make this viewpoint precise without committing to a particular governing equation, we encode such principles abstractly through a structural functional $\mathcal C : \mathcal U \to \mathbb R^k$, acting on the state space $\mathcal U$. Along an evolution $u_t$, the quantities $\mathcal C(u_t)$ describe the physical structures to be preserved or controlled over time. For differentiable $\mathcal C$, examples include conserved quantities and dissipative structures satisfying
\begin{equation}\label{eq:constraints}
    \frac{\mathrm d}{\mathrm dt}\mathcal C(u_t)=0,
    \quad \text{or} \quad
    \frac{\mathrm d}{\mathrm d t}\mathcal C(u_t)\le 0.
\end{equation}
The equality case represents preserved relations, such as conservation laws or invariant manifolds. The inequality case represents admissibility or monotonicity properties, such as positivity, boundedness, entropy inequalities, or energy dissipation. Thus, as discussed above, $\mathcal C$ is not meant to determine the full dynamics; rather, it records selected structural information that any physically meaningful evolution should satisfy.

For continuous-time dynamics, an invariant $\mathcal C(u_t)=0$ also has the local form
\begin{equation}\label{eq:constraint-chain-rule}
    0
    =
    \frac{\mathrm d}{\mathrm dt}\mathcal C(u_t)
    =
    D\mathcal C(u_t)\big(\mathcal D(u_t)\big),
    \qquad u'=\mathcal D(u).
\end{equation}
Thus, structural functionals constrain the admissible evolution: the vector
field must be tangent to the level sets of $\mathcal C$. Since Latent Twins are
formulated at the level of solution operators, we do not require access to
$\mathcal D$. Instead, we ask whether the surrogate solution operator preserves
the structural relation encoded by $\mathcal C$.

We first treat this question a posteriori. The approximation theorem of
\cite{chung2026latent} gives uniform control of the lifted Latent Twin surrogate
on a compact invariant set. Consequently, any Lipschitz structural functional
evaluated on the surrogate remains close to its value along the exact flow. We
later show that sharper estimates are possible when the structure is imposed
directly in latent space.

\paragraph{Assumption on PDEs and Latent Twin.}
Let $(\mathcal{U},\|\cdot\|_{\mathcal{U}})$ be a Banach space, and let $\Phi(t,s,\cdot)$ and $\Phi_N(t,s,\cdot)$ denote the exact PDE solution operator and corresponding discrete flow on $[T_0,T_{\mathrm{end}}]$, respectively. We make the following assumptions:
\smallskip\noindent
\textbf{(A1) Well-posed flow.}
Assume there exists a compact invariant set $\mathcal{K}\subset \mathcal{U}$ such that
$$
\Phi(t,s,\cdot):\mathcal{K}\to\mathcal{K}
$$
is well defined for all $s,t\in[T_0,T_{\mathrm{end}}]$, and is Lipschitz continuous in the initial condition on $[T_0,T_{\mathrm{end}}]^2\times\mathcal{K}$ with constant $L_{\mathcal{K}}>0$.

\smallskip\noindent
\textbf{(A2) Stable and consistent spatial discretization.}
Let $P_N:\mathcal{U}\to\mathbb{R}^N$ and $R_N:\mathbb{R}^N\to\mathcal{U}$ be uniformly bounded projection and lifting operators satisfying
$$
\varepsilon_{\mathrm{disc}} := \sup_{u\in\mathcal{K}} \|R_N(P_N(u))-u\|_{\mathcal {U}}\to 0 \qquad \text{for } N\to\infty,
$$
and assume the discrete flow $\Phi_N$ is consistent with the exact flow in the sense that for some constant $C_{\mathrm{disc}}>0$,
$$
\sup_{\substack{u\in\mathcal{K}\\ t,s\in[T_0,T_{\mathrm{end}}]}} \|R_N(\Phi_N(t,s,P_N(u))) - \Phi(t,s,u)\|_{\mathcal{U}} \leq C_{\mathrm{disc}} \, \varepsilon_{\mathrm{disc}}.
$$

\smallskip \noindent
\textbf{(A3) Expressiveness of Latent Twin.}
On the compact discrete state set
$$
\mathcal{K}_N:=P_N(\mathcal{K})\subset\mathbb{R}^N,
$$
assume the encoder $e:\mathbb{R}^N\to Z$, decoder $d:Z\to\mathbb{R}^N$, and latent map $m:[T_0,T_{\mathrm{end}}]^2\times Z \to Z$ satisfy:
\begin{enumerate}
    \item $e$ and $d$ are Lipschitz continuous, with decoder Lipschitz constant $L_{d}$;
    \item the reconstruction error obeys
    $$
    \sup_{x\in\mathcal{K}_N} \|d(e(x))-x\|
    \leq \varepsilon_{\mathrm{ae}},
    $$
    where $\varepsilon_{\mathrm{ae}} \geq 0$ denotes the autoencoder reconstruction error;
    \item $m$ approximates the exact discrete latent flow
    $$
    m^\ast_{s\to t}(z):= e\bigl(\Phi_N(t,s,d(z))\bigr)
    $$ uniformly, with
    $$
    \sup_{(t,s,z)\in [T_0,T_{\mathrm{end}}]^2\times Z} \|m_{s\to t}(z)-m^\ast_{s\to t}(z)\| \leq \varepsilon_{\mathrm{map}},
    $$
    where $\varepsilon_{\mathrm{map}} \geq 0$ is the latent map approximation error.
\end{enumerate}

\begin{theorem}[Latent Twin Approximation Bound {\cite[Theorem~2]{chung2026latent}}]
\label{thm:lt-pde}
For $u_s\in\mathcal{K}$, define the lifted Latent Twin surrogate by
$$
\Psi_{N}(t,s,u_s) :=
R_N\Bigl( d\bigl(m_{s\to t} (e(P_N(u_s)))\bigr)\Bigr).
$$
Then, under \textup{(A1)--(A3)}, there exists a constant $C_{\mathrm{lt}}>0$ such that
$$
\sup_{(s,t,u_s)\in[T_0,T_{\mathrm{end}}]^2\times\mathcal{K}}
\left\| \Psi_{N}(t, s,u_s) -\Phi(t,s,u_s) \right\|_{\mathcal U}
\leq
C_{\mathrm{lt}}
\left(
\varepsilon_{\mathrm{disc}}
+
\bigl(1+\mathrm{e}^{L_{\mathcal K}(T_{\mathrm{end}}-T_0)}\bigr)\varepsilon_{\mathrm{ae}}
+
L_{d}\varepsilon_{\mathrm{map}}
\right).
$$
In particular, if
$$
\varepsilon_{\mathrm{disc}}\to 0, \qquad
\varepsilon_{\mathrm{ae}}\to 0, \qquad
\varepsilon_{\mathrm{map}}\to 0, \qquad \text{for } \quad N \to \infty
$$
then
$$
\Psi_{N}\to \Phi
\qquad
\text{uniformly on } [T_0,T_{\mathrm{end}}]^2\times\mathcal{K}.
$$
\end{theorem}

The proof is provided in \cite{chung2026latent}, using the ODE Latent Twin approximation theorem as an intermediate step on the discretized system. In the finite-dimensional setting, $\Psi_N$ reduces to the latent twin $\Psi$ defined in \Cref{eq:latenttwinfinite}.

\section{Physics-conforming Latent Twins}\label{sec:method}

Physics-conforming Latent Twins build on the solution-operator formulation by incorporating structural constraints into the latent evolution itself. We first show that fully data-driven Latent Twins inherit physical structure a posteriori: structural defects are controlled by the surrogate approximation error; see \Cref{sec:posteriorConform}. We then introduce and prove in \Cref{sec:aprioriConform} that enforcing compatible constraints directly in latent space yields sharper estimates. Finally, \Cref{sec:pclflow} develops tractable latent-flow constructions that realize preservation or dissipation properties by design.

\subsection{A posteriori physics-conforming estimate}\label{sec:posteriorConform}
The approximation theorem above immediately yields a corresponding control of structural defect. If the surrogate trajectory is uniformly close to the true trajectory, then any Lipschitz structural functional evaluated on the surrogate must also remain close to its exact value. This gives a direct and very general \emph{a posteriori} estimate.

For notational clarity, we state the estimates for scalar structural functionals. Vector-valued constraints are handled componentwise, as noted after \Cref{thm:apriori}.

\begin{theorem}[Constraint-consistent Latent Twin Approximation Bound]
\label{thm:constraint-consistent}
Under the assumptions of \Cref{thm:lt-pde}, let $\mathcal{C}:\mathcal{U}\to\mathbb{R}$ be Lipschitz continuous on a set containing the exact and surrogate trajectories, with Lipschitz constant $L_{\mathcal{C}}>0$. Assume that the exact flow satisfies the structural property
$$
\mathcal{C}\bigl(\Phi(t,s,u_s)\bigr)\leq 0
\qquad \text{for all } (t,s,u_s)\in[T_0,T_{\mathrm{end}}]^2\times\mathcal{K}.
$$
Let $[y]_+ := \max(y,0)$. Then the lifted Latent Twin surrogate satisfies
$$
\sup_{(t,s,u_s)\in[T_0,T_{\mathrm{end}}]^2\times\mathcal{K}}
\left\|
\left[
\mathcal{C}\bigl(\Psi_{N}(t,s,u_s)\bigr)
\right]_+
\right\|
\leq
L_{\mathcal C} C_{\mathrm{lt}}
\left(
\varepsilon_{\mathrm{disc}}
+
\bigl(1+\mathrm{e}^{L_{\mathcal K}(T_{\mathrm{end}}-T_0)}\bigr)\varepsilon_{\mathrm{ae}}
+
L_{d}\varepsilon_{\mathrm{map}}
\right).
$$
\end{theorem}

\begin{proof}
Fix $(t,s,u_s)\in[T_0,T_{\mathrm{end}}]^2\times\mathcal{K}$. Since the exact flow satisfies $\mathcal{C}\bigl(\Phi(t,s,u_s)\bigr)\leq 0$, we have
\[
    \left[ \mathcal{C}\bigl(\Psi_{N}(t,s,u_s)\bigr) \right]_+ \leq \left| \mathcal{C}\bigl(\Psi_{N}(t,s,u_s)\bigr) - \mathcal{C}\bigl(\Phi(t,s,u_s)\bigr) \right|.
\]
By Lipschitz continuity of $\mathcal C$,
\[
    \left| \mathcal{C}\bigl(\Psi_{N}(t,s,u_s)\bigr) - \mathcal{C}\bigl(\Phi(t,s,u_s)\bigr) \right| \leq L_{\mathcal C} \left\| \Psi_{N}(t,s,u_s)-\Phi(t,s,u_s) \right\|_{\mathcal U}.
\]
Taking the supremum proves the estimate by using \Cref{thm:lt-pde}.
\end{proof}

\begin{corollary}[Equality-type structural properties]
\label{cor:equality-constraint}
If instead the exact flow satisfies
$$
\mathcal{C}\bigl(\Phi(t,s,u_s)\bigr)=0
\qquad \text{for all } (t,s,u_s)\in[T_0,T_{\mathrm{end}}]^2\times\mathcal{K},
$$
then
$$
\sup_{(t,s,u_s)\in[T_0,T_{\mathrm{end}}]^2\times\mathcal{K}}
\left\|
\mathcal{C}\bigl({\Psi}_{N}(t,s,u_s)\bigr)
\right\|
\leq
2L_{\mathcal{C}}C_{\mathrm{lt}}
\left(
\varepsilon_{\mathrm{disc}}
+
\bigl(1+\mathrm{e}^{L_{\mathcal K}(T_{\mathrm{end}}-T_0)}\bigr)\varepsilon_{\mathrm{ae}}
+
L_{d}\varepsilon_{\mathrm{map}}
\right).
$$
\end{corollary}

\begin{proof}
Apply \Cref{thm:constraint-consistent} to both $\mathcal{C}$ and $-\mathcal{C}$. The claim follows from the elementary estimate
$$
\|\mathcal{C}(x)\|
\leq
\bigl\|[\mathcal C(x)]_+\bigr\|
+
\bigl\|[-\mathcal C(x)]_+\bigr\|.
$$
\end{proof}

\paragraph{Limitations of the a posteriori estimate.}
The estimates above show that structural defect can be controlled directly from the Latent Twin approximation error. However, this control remains purely \emph{a posteriori}: it is obtained by passing the full trajectory error through the Lipschitz continuity of $\mathcal{C}$, so it inherits the entire approximation structure of the surrogate itself, including discretization, autoencoder, and latent-map errors, as well as the growth factor $\mathrm e^{L_{\mathcal K}(T_{\mathrm{end}}-T_0)}$. Consequently, even exact structure preservation by the latent map would not yet appear in the bound. This motivates a sharper, \emph{a priori} approach in which structural properties are enforced directly in latent space.

\subsection{A priori physics-conforming estimate}\label{sec:aprioriConform}

\paragraph{Latent-space structural enforcement.}
The preceding estimate treats structure only after the surrogate trajectory has been produced. A sharper approach is to represent physical constraints directly in latent coordinates and enforce it at the level of the latent propagator itself. To this end, we introduce a latent constraint functional
\[
\mathcal C_Z : Z \to \mathbb R^k,
\]
intended to encode the same structural property in latent space that $\mathcal C$ encodes in the original state space. If such a latent constraint is preserved exactly by the latent evolution, then structural defect need no longer be routed through the full trajectory approximation error. Instead, it can be controlled directly through the mismatch between $\mathcal C$ and $\mathcal C_Z$ and the error introduced by encoding and decoding.

Throughout, we denote by
\[
\mathcal{Z} := \bigl\{ m_{s\to t}(e(P_N(u_s))) :
u_s\in\mathcal K,\;
s\le t\in[T_0,T_{\mathrm{end}}]
\bigr\}
\subset Z
\]
the set of latent states reachable by the latent map from encoded conditions at times $s$ in $\mathcal K$.

\begin{definition}[Compatible latent constraint]
\label{def:compatible}
Let $\mathcal{C} : \mathcal{U} \to \mathbb{R}^k$ be a constraint functional, $d : Z \to \mathbb{R}^N$ a discrete decoder, and $R_N : \mathbb{R}^N \to \mathcal{U}$ the lifting operator. A latent constraint
$\mathcal{C}_Z : Z \to \mathbb{R}^k$ is $\varepsilon_{\mathcal{C}}$ \emph{-compatible} with $\mathcal{C}$ via $R_N \circ d$ with defect
\[
\varepsilon_{\mathcal{C}}
:=
\sup_{z \in \mathcal{Z}}
\|\mathcal{C}(R_N(d(z))) - \mathcal{C}_Z(z)\|,
\]
provided $\varepsilon_{\mathcal{C}} < \infty$.
\end{definition}

The defect $\varepsilon_{\mathcal C}$ quantifies how faithfully the latent constraint $\mathcal C_Z$ represents the physical constraint after decoding and lifting. In this sense, compatibility is the mathematical bridge between structural properties in state space and their latent counterparts. When $\mathcal C$ is known explicitly, the choice $\mathcal C_Z = \mathcal C \circ R_N \circ d$ yields $\varepsilon_{\mathcal C}=0$ by definition. In practice, however, this pullback is typically unavailable before training, since the decoder is learned, and moreover may be too complicated to preserve directly through explicit constraints on the latent propagator. One therefore prescribes $\mathcal C_Z$ from a tractable family and uses $\varepsilon_{\mathcal C}$ to measure the resulting mismatch, see following discussions in \Cref{sec:pclflow}.

\begin{definition}[Structure-preserving latent flow]
\label{def:structure-preserving}
Let $\mathcal{C}_Z : Z \to \mathbb{R}^k$ be a latent constraint. The latent flow $m$ is \emph{$\mathcal{C}_Z$-preserving} if for all $z \in \mathcal{Z}$ and $s \leq t \in [T_0,T_{\mathrm{end}}]$:
\begin{enumerate}
    \item[(i)] \emph{(inequality constraint)}
    \[
    \mathcal{C}_Z\bigl(m_{s\to t}(z)\bigr) \leq \mathcal{C}_Z(z);
    \]
    \item[(ii)] \emph{(equality constraint)}
    \[
    \mathcal{C}_Z\bigl(m_{s\to t}(z)\bigr) = \mathcal{C}_Z(z),
    \]
    expressing exact invariance of $\mathcal{C}_Z$ along latent trajectories.
\end{enumerate}

\end{definition}

Condition (ii) is the natural invariance requirement for a conservation law: the latent constraint is preserved exactly, mirroring $\mathcal{C}(u(t)) = 0$ in physical space.
Condition (i) is stronger than forward-invariance of the admissible set $\{\mathcal C_Z\leq 0\}$: it requires $\mathcal C_Z$ to be monotonically non-increasing along reachable latent trajectories.

This monotonicity is needed in the proof of \Cref{thm:apriori} to control $[\mathcal C_Z(z_{s\to t})]_+$ even when the encoded initial state $z_s$ is not exactly in the admissible set, which can happen when compatibility or autoencoder errors are nonzero. For dissipative systems such as the heat equation, this condition is physically natural: $\mathcal{C}_Z$ plays the role of a latent energy that the dynamics dissipate.

\begin{theorem}[Structure-preserving Latent Twin Bound]
\label{thm:apriori}
Let $\mathcal{C} : \mathcal{U} \to \mathbb{R}$ be Lipschitz with constant $L_{\mathcal{C}} > 0$, and suppose the true flow satisfies $\mathcal{C}(\Phi(t,s,u_s)) \leq 0$ for all $(t,s, u_s) \in [T_0,T_{\mathrm{end}}]^2 \times \mathcal K$. Let  $\mathcal{C}_Z$ be compatible with $\mathcal{C}$ via $R_N \circ d$ with defect $\varepsilon_{\mathcal{C}}$ in the sense of \Cref{def:compatible}, and suppose the latent map $m$ is $\mathcal{C}_Z$-preserving  in the sense of \Cref{def:structure-preserving}. Then
\[
\sup_{(t,s,u_s)\in[T_0,T_{\mathrm{end}}]^2\times \mathcal K} \Bigl\|\Bigl[ \mathcal{C}\Bigl( \Psi_{N} (t,s,u_s) \Bigr) \Bigr]_+\Bigr\| \le 2\varepsilon_{\mathcal{C}} + L_{\mathcal{C}}\,\varepsilon_{\mathrm{ae}}.
\]
\end{theorem}

\begin{proof}
Let $z_s := e(P_N(u_s))$ denote the encoded state at time $s$ and
$z_{s \to t} := m_{s\to t}(z_s)$ the surrogate latent state propagated to time $t$, so that
$\Psi_{N}(t,s,u_s) = R_N(d(z_{s \to t}))$. Insert
$\mathcal{C}_Z(z_{s \to t})$ to decompose the physical constraint violation:
\begin{align*}
\bigl[\mathcal{C}(R_N(d(z_{s \to t})))\bigr]_+
&\leq
\bigl[\mathcal{C}(R_N(d(z_{s \to t})))
- \mathcal{C}_Z(z_{s \to t})\bigr]_+
+ \bigl[\mathcal{C}_Z(z_{s \to t})\bigr]_+.
\end{align*}
The first term is bounded by $\varepsilon_{\mathcal{C}}$
by \Cref{def:compatible}. For the second, since  $m$ is $\mathcal{C}_Z$-preserving and $s \leq t$,
\[
    \mathcal{C}_Z(z_{s \to t})  \leq \mathcal{C}_Z(z_s).
\]
It remains to bound the initial latent constraint
value. Writing
$\hat{u} := R_N(d(z_s)) \in \mathcal{U}$,
which satisfies
$\|\hat{u} - u_s\|_{\mathcal{U}} \leq
\varepsilon_{\mathrm{ae}}$, and inserting
$\mathcal{C}(u_s) \leq 0$,
\begin{align*}
\mathcal{C}_Z(z_s)
&\leq
\bigl\|\mathcal{C}_Z(z_s)
- \mathcal{C}(R_N(d(z_s)))\bigr\|
+ \bigl\|\mathcal{C}(\hat{u})
- \mathcal{C}(u_s)\bigr\|
+ \mathcal{C}(u_s).
\end{align*}
The first term is bounded by $\varepsilon_{\mathcal{C}}$
by \Cref{def:compatible}, the second by
$L_{\mathcal{C}}\|\hat{u} - u_s\|_{\mathcal{U}}
\leq L_{\mathcal{C}}\,\varepsilon_{\mathrm{ae}}$,
and $\mathcal{C}(u_s) \leq 0$. Taking the positive part, combining, and taking the supremum over
$(t,s,u_s) \in [T_0,T_{\mathrm{end}}]^2 \times \mathcal K$ yields the stated bound.
\end{proof}
\begin{remark}\label{rem:comparison}
\Cref{thm:apriori} is sharper than the a posteriori estimate of \Cref{thm:constraint-consistent} since it avoids routing structural defect through the full trajectory approximation error. The a posteriori estimate first bounds the surrogate trajectory error and then passes this error through the Lipschitz continuity of $\mathcal C$. As a result, it inherits the complete approximation structure of the surrogate, including discretization error, latent-flow error, and the growth factor $\mathrm e^{L_{\mathcal K}(T_{\mathrm{end}}-T_0)}$ from the underlying solution-operator approximation theorem.

By contrast, when a compatible latent constraint is enforced directly by the latent flow, the latent evolution itself contributes no additional structural violation. The remaining defect is controlled only by the compatibility error between the physical constraint $\mathcal C$ and its latent counterpart $\mathcal C_Z$, together with the decoder/autoencoder error. In particular, when $\varepsilon_{\mathcal C}=0$, the bound reduces to
\[
\sup_{(t,s,u_s)\in[T_0,T_{\mathrm{end}}]^2\times \mathcal K}
\Bigl\|\Bigl[
\mathcal{C}\Bigl(
\Psi_{N}(t,s,u_s)
\Bigr)
\Bigr]_+\Bigr\|
\le
L_{\mathcal{C}}\,\varepsilon_{\mathrm{ae}}.
\]
Thus, in the ideal compatible case, only the autoencoding error remains. No additional amplification through the latent evolution appears in the structural bound. This provides a theoretical justification for prescribing and enforcing suitable constraints directly in latent space, rather than relying only on a posteriori control through trajectory approximation.
\end{remark}

\begin{remark}
The estimates are stated for scalar structural functionals to keep the notation and proofs transparent. Vector-valued constraints $\mathcal C=(\mathcal C_1,\ldots,\mathcal C_k)$ are handled componentwise by applying the scalar estimates to each component $\mathcal C_i$ and combining the resulting bounds in the chosen norm.
\end{remark}

Conservation and dissipation laws, expressed as $\mathcal C(u(t))=\mathcal C(u(s))$ or $\mathcal C(u(t))\leq \mathcal C(u(s))$ for $s\leq t$, fit the preceding estimates after recentering relative to the source state.

\subsection{Compatible pullbacks and physics-conforming latent flow} \label{sec:pclflow}

The bound of \Cref{thm:apriori} shows that physics conformity is not a post-processing property of an accurate trajectory surrogate, but a design principle for the latent twin itself. The design of $\mathcal C_Z$ is therefore inseparable from the design of $m_{s\to t}$: $\mathcal C_Z$ should approximate the decoder-induced pullback of the physical quantity $\mathcal C$, while $m_{s\to t}$ should belong to a class that preserves or dissipates $\mathcal C_Z$ by construction or with controlled violation. Thus the admissible latent constraints are determined simultaneously by the geometry of the reconstruction map and by the algebraic or geometric structure of the latent flow.

\begin{figure}
\centering
\includegraphics[width=0.9\linewidth]{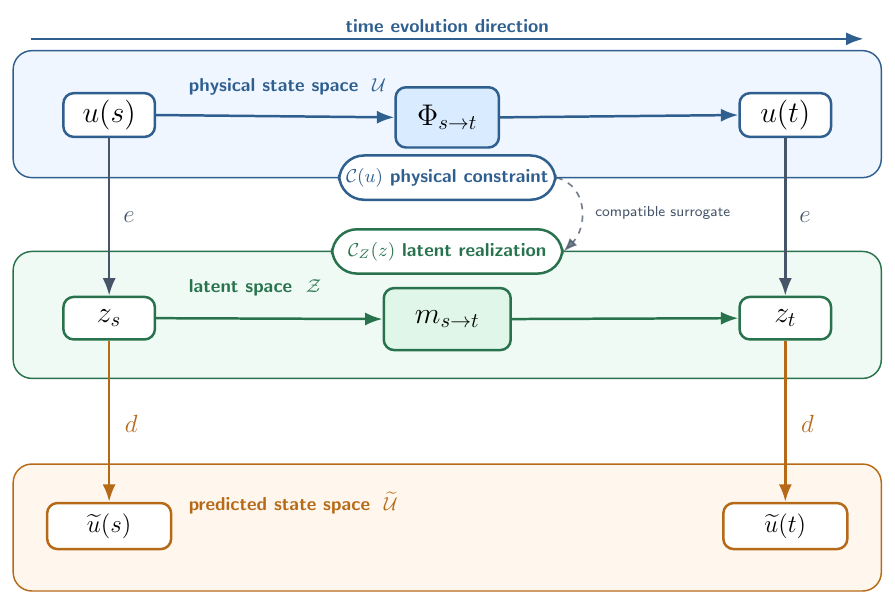}
\caption{Schematic of a Physics-conforming Latent Twin. The encoder $e$ maps a state $u(s)\in\mathcal U$ to a latent state $z_s\in\mathcal Z$, the latent flow $m_{s\to t}$ propagates $z_s$ to $z_t$, and the decoder $d$ maps the latent states back to reconstructed or predicted states. The structural functional $\mathcal C$ encodes a physical property in the state space, for example a conserved quantity or a dissipative functional. A compatible latent constraint $\mathcal C_Z$ realizes this structure in latent space, where it can be preserved or dissipated directly by the latent flow. In this way, physics conformity is imposed on the latent evolution itself.}
\label{fig:latent_twin_schematic}
\end{figure}

\paragraph{Latent constraints.}
The \emph{exact} latent constraint is the decoder-induced \emph{pullback}
\begin{equation}\label{eq:pullback}
    \mathcal C_Z^\star  := \mathcal C\circ R_N\circ d.
\end{equation}
With this choice, the latent constraint agrees exactly with the physical constraint after decoding, so that $\varepsilon_{\mathcal C}=0$ and the compatibility term in the structure-preserving bound vanishes; see \Cref{rem:comparison}. In practice, however, the pullback $\mathcal C_Z^\star$ is not known \emph{a priori}, since the decoder is learned together with the latent representation. Even when $\mathcal C$ is known, its learned pullback $\mathcal C\circ R_N\circ d$ is generally not in a form that the latent flow can preserve exactly. One therefore replaces $\mathcal C_Z^\star$ by a tractable latent functional $\mathcal C_Z$ that should satisfy two complementary requirements.

Note, equivalently, compatibility may also be assessed from the encoder side by requiring $\mathcal C_Z(e(u)) \approx \mathcal C(u)$, for states $u$ on the data manifold. This provides a data-manifold formulation of the same representation requirement. In what follows, we use the decoder-side formulation in \Cref{eq:pullback}, since it is directly tied to the physical states generated by the Latent Twin.

Thus, $\mathcal C_Z$ should approximate the exact pullback \Cref{eq:pullback} in a form that the latent flow can preserve or dissipate. For a conservative structure, this means
\[
    \mathcal C_Z(m_{s\to t}(z)) = \mathcal C_Z(z),
\]
whereas for a dissipative structure one seeks
\[
    \mathcal C_Z(m_{s\to t}(z)) \le \mathcal C_Z(z), \qquad s\le t .
\]
Physics conformity therefore depends on both sides of the construction: how well $\mathcal C_Z$ matches the decoded physical quantity, and how consistently the latent flow propagates $\mathcal C_Z$. This relationship is illustrated in \Cref{fig:latent_twin_schematic}.

In the linear setting, this connection can be made explicit. Linear reconstruction maps preserve the algebraic form of linear and quadratic physical functionals under pullback, while linear latent flows reduce preservation and dissipation to finite-dimensional operator conditions.

\begin{proposition}[Exact pullback realization of linear and quadratic constraints] \label{prop:affine_linear_exact_constraints}
Let $Z$ be a finite-dimensional latent space and let $\mathcal U$ be the physical state space. Suppose that the reconstruction map is linear, $R_N(d(z))=\mathcal D z$, where $\mathcal D:Z\to\mathcal U$ is a bounded linear lifting operator, and let the latent flow be linear, $m_{s\to t}(z)=A_{s,t}z$.

\begin{enumerate}
    \item If $\mathcal C:\mathcal U\to\mathbb R$ is a continuous linear functional, then there exists a vector $a$ such that the decoder-induced pullback is linear, $\mathcal C_Z^\star(z)=a^\top z$. It is preserved by the latent flow for all $z\in Z$ if and only if $A_{s,t}^\top a=a$.

    \item If $\mathcal C(u)=B_{\mathcal U}(u,u)$ with $B_{\mathcal U}:\mathcal U\times\mathcal U\to\mathbb R$ a continuous symmetric bilinear form, then there exists a symmetric matrix $Q_Z$ such that the decoder-induced pullback satisfies $\mathcal C_Z^\star(z)=z^\top Q_Zz$. It is preserved by the latent flow for all $z\in Z$ if and only if $A_{s,t}^\top Q_ZA_{s,t}=Q_Z$, and it is dissipated for $s\le t$ if and only if $A_{s,t}^\top Q_ZA_{s,t}\preceq Q_Z$.
\end{enumerate}
\end{proposition}

\begin{proof}
We prove the two cases separately.

\begin{enumerate}
    \item Since $\mathcal C$ is linear and $\mathcal D$ is linear, the
    decoder-induced pullback $\mathcal C_Z^\star(z)=\mathcal C(\mathcal D z)$ is a linear functional on the finite-dimensional space $Z$. Hence there exists a vector $a$ such that $\mathcal C_Z^\star(z)=a^\top z$. Preservation under $m_{s\to t}(z)=A_{s,t}z$ is equivalent to $a^\top A_{s,t}z=a^\top z$ for all $z\in Z$, which holds if and only if $A_{s,t}^\top a=a$.

    \item Since $B_{\mathcal U}$ is symmetric bilinear and $\mathcal D$ is linear, the decoder-induced pullback $\mathcal C_Z^\star(z)=B_{\mathcal U}(\mathcal D z,\mathcal D z)$ is a quadratic form on the finite-dimensional space $Z$. Therefore, after choosing coordinates on $Z$, there exists a symmetric matrix $Q_Z$ such that $\mathcal C_Z^\star(z)=z^\top Q_Zz$. Preservation is equivalent to $z^\top A_{s,t}^\top Q_ZA_{s,t}z=z^\top Q_Zz$ for all $z\in Z$, which holds if and only if $A_{s,t}^\top Q_ZA_{s,t}=Q_Z$. The dissipative condition follows similarly from $z^\top(A_{s,t}^\top Q_ZA_{s,t}-Q_Z)z\le 0$ for all $z\in Z$.
\end{enumerate}
\end{proof}

\begin{remark}[Scope of the exact realization result]
The preceding proposition is stated in a simplified setting to isolate the basic correspondence between physical structure, latent constraints, and latent-flow invariance.

\begin{enumerate}
    \item The proposition is stated for a single scalar constraint $\mathcal C:\mathcal U\to\mathbb R$ to keep the notation light. In practice, many systems possess several simultaneous structural principles, such as mass and energy conservation in fluid dynamics, mass and positivity constraints in population dynamics, or conservation laws together with entropy inequalities in kinetic equations. As noted above the framework extends directly to vector-valued constraints $\mathcal C:\mathcal U\to\mathbb R^k$ and $\mathcal C_Z:Z\to\mathbb R^k$. The corresponding compatibility defects, operator constraints, and trajectory-level penalties may then be imposed componentwise or through suitable weighted combinations.

    \item Affine decoder offsets and affine or quadratic-affine constraint terms are omitted for notational clarity. They can be incorporated without changing the underlying principle by recentering or by augmenting the latent state with a constant coordinate.

    \item The conditions have a direct invariance interpretation: $A_{s,t}^\top a=a$ means that $a^\top z$ is invariant under the latent flow, while $A_{s,t}^\top Q_ZA_{s,t}=Q_Z$ means that $z^\top Q_Zz$ is invariant. For example, $a=[1,\ldots,1]^\top$ corresponds to preservation of the sum of latent coordinates, analogous to total mass, whereas $Q_Z=I$ corresponds to preservation of the Euclidean latent norm, analogous to a latent energy.

    \item The linear case is stated only for conservation. On a full vector space, a global one-sided inequality for a scalar linear functional reduces to equality; nontrivial monotonicity requires restricting the dynamics to an admissible ordered set, such as a cone or positivity-preserving state space.
\end{enumerate}
\end{remark}

For nonlinear neural-network decoders, the proposition should be read as a design principle rather than a global guarantee. Linear and quadratic physical functionals still motivate affine or quadratic choices of $\mathcal C_Z$, but these choices must be related to the learned decoder. This motivates pullback-matching regularization in addition to latent conformity: the former connects $\mathcal C_Z$ to the decoded physical quantity, while the latter enforces preservation or dissipation of $\mathcal C_Z$ along the latent flow.

\paragraph{Latent flow parameterizations.}
The preceding discussion identifies which latent constraints are natural from the reconstruction side. We now briefly describe how different choices of the latent flow support their enforcement.

A general neural parameterization of the latent flow $m_{s\to t}$ is the most expressive choice. In this case, conservation or dissipation is usually imposed at the trajectory-level, for example by penalizing deviations of $\mathcal C_Z(m_{s\to t}(z_s))$ from $\mathcal C_Z(z_s)$, or by using a one-sided penalty for dissipative constraints. This allows flexible dynamics, but exact preservation is generally not guaranteed away from the sampled training pairs.

A useful reference case for operator-level enforcement is the exponential latent flow
\[
    m_{s\to t}(z)=\exp((t-s)W)z .
\]
This form preserves the semigroup structure, $m_{r\to t}\circ m_{s\to r}=m_{s\to t}$ for $s\le r\le t$, and turns structure preservation into generator-level conditions. For example, conservation of a linear latent quantity $a^\top z$ corresponds to $a^\top W=0$, while conservation or dissipation of a quadratic latent quantity $z^\top Q_Zz$ corresponds to conditions on $W^\top Q_Z+Q_ZW$. These operator-level conditions are summarized in \Cref{tab:latent_constraint_classes}.

The same time-consistency idea can be combined with nonlinear coordinates by conjugation. Let $V:Z\to Z$ be an invertible neural network and let $\varphi_{s\to t}$ be a simpler latent flow satisfying $\varphi_{r\to t}\circ\varphi_{s\to r}=\varphi_{s\to t}$. Then
\[
    m_{s\to t}
    =
    V\circ \varphi_{s\to t}\circ V^{-1}
\]
inherits the same composition property. Constraints may be imposed in the transformed coordinate and pulled back through $V^{-1}$, allowing nonlinear latent dynamics while retaining an analytically controlled inner flow. In particular, a quadratic invariant in the transformed coordinate generally becomes a nonlinear invariant in the original latent coordinate.

The constraint may be enforced at different levels. At the architectural level, the latent flow is parameterized so that the desired property holds identically, for example by choosing a skew-symmetric or dissipative generator. At the operator-regularization level, one penalizes algebraic violations such as $a^\top W$ or $W^\top Q_Z+Q_ZW$. At the trajectory level, one penalizes violations evaluated on sampled latent pairs, such as $\mathcal C_Z(m_{s\to t}(z_s))-\mathcal C_Z(z_s)$. These choices trade expressiveness against exactness: architectural constraints give the strongest guarantees, while trajectory-level penalties are more flexible but only enforce the property on sampled data.

\Cref{tab:latent_constraint_classes} summarizes the preceding design choices by listing representative physical functionals, compatible latent constraint classes, and corresponding operator- or trajectory-level realizations.

\newcolumntype{L}[1]{>{\raggedright\arraybackslash\bfseries}m{#1}}
\newcolumntype{M}[1]{>{\centering\arraybackslash}m{#1}}

\begin{table}[ht!]
\centering
\caption{Representative physical structures and compatible latent-space realizations. Each column displays one constraint type in isolation, although practical models may combine several structures. The examples are written in PDE notation, where physical quantities are naturally expressed as spatial functionals over a domain $\Omega$, but the same linear, quadratic, and monotone structures also arise for ODEs, iterative optimization dynamics, and other evolutionary or sequential systems. ``Lyapunov-type'' refers to a functional whose value is nonincreasing along admissible trajectories. Operator-level realizations are shown for the exponential linear latent flow $m_{s\to t}(z)=\exp((t-s)W)z$; other latent-flow classes may admit analogous structure-preserving parameterizations. In the enstrophy example, $\omega=\nabla\times u$ denotes the vorticity of the velocity field.}
\label{tab:latent_constraint_classes}

\renewcommand{\arraystretch}{1.4}
\small

\definecolor{LTLabel}{RGB}{242,242,242}
\definecolor{LTLinear}{RGB}{232,242,255}
\definecolor{LTQuadratic}{RGB}{235,247,238}
\definecolor{LTMonotone}{RGB}{255,244,230}

\begin{tabular}{%
>{\columncolor{LTLabel}}L{2.0cm}
>{\columncolor{LTLinear}}M{3.5cm}
>{\columncolor{LTQuadratic}}M{4.0cm}
>{\columncolor{LTMonotone}}M{4.5cm}}
\toprule

\rowcolor{LTLabel}
&
\cellcolor{LTLinear}\textbf{Linear}
&
\cellcolor{LTQuadratic}\textbf{Quadratic}
&
\cellcolor{LTMonotone}\textbf{Monotone}
\\

\midrule
Example physical functional
&
\[\mathcal C(u)=\int_\Omega u(\xi)\,\mathrm d\xi\]
&
\makecell[c]{$\displaystyle\mathcal C(u)=\int_\Omega |u(\xi)|^2\,\mathrm d\xi$
\\[1em]
$\displaystyle\mathcal C(u)=
\int_\Omega |\omega(\xi)|^2\,\mathrm d\xi$}
&
\[\mathcal C(u) = \int_\Omega u(\xi)\log u(\xi)\,\mathrm d\xi\]
\\

\addlinespace[0.4em]
Desired property
&
\[\mathcal C(u(t))=\mathcal C(u(s))\]
&
\[\mathcal C(u(t))=\mathcal C(u(s))\]
&
\[\mathcal C(u(t))\le\mathcal C(u(s)) \text{ for } s\le t\]
\\

\addlinespace[0.4em]
Compatible latent constraint
&
\[\mathcal C_Z(z)=a^\top z+b\]
&
\[\mathcal C_Z(z)=z^\top Q_Z z+a^\top z+b\]
&
\makecell[c]{$\mathcal C_Z: Z\to\mathbb{R}$\\Lyapunov-type\\
latent functional
}

\\
\addlinespace[0.4em]
Operator-level realization
&
\[a^\top W = 0\]
&
\[W^\top Q_Z + Q_ZW = 0\]
&
\[W + W^\top \preceq 0\]
\\

\addlinespace[0.4em]
Trajectory-level realization
&
\[\mathcal C_Z(m_{s\to t}(z_s))=\mathcal C_Z(z_s)\]
&
\[\mathcal C_Z(m_{s\to t}(z_s))=\mathcal C_Z(z_s)\]
&
\[\mathcal C_Z(m_{s\to t}(z_s))\le\mathcal C_Z(z_s)\text{ for }s\le t\]
\\

\bottomrule
\end{tabular}

\end{table}

\section{Numerics}\label{sec:numerics}

The numerical experiments are designed to assess how Physics-conforming Latent Twins encode and preserve structure across representative ODE and PDE benchmarks. The goal is not to provide an exhaustive comparison, but to demonstrate how different physical structures can be incorporated into latent solution operators and how such constraints affect prediction accuracy, stability, extrapolation, and qualitative physical consistency.

We first introduce the shared training setup, notation, and loss formulation in \Cref{sec:numsetup}. The ODE experiments in \Cref{sec:ODEs} proceed from linear to nonlinear and dissipative dynamics. In the linear setting, we compare Latent Twins with POD and DMD on stable oscillatory systems. For nonlinear dynamics, we study the multigroup SIR model and the undamped pendulum, which illustrate soft enforcement of a linear conservation law and hard enforcement of a quadratic latent invariant, respectively. The pendulum example also includes a comparison with a physics-informed neural network baseline. We then consider gradient-flow dynamics from optimization as a dissipative benchmark. Finally, \Cref{sec:PDEs} turns to canonical heat and wave equations, which test the framework on high-dimensional spatially discretized PDEs with dissipative and conservative structure.

\subsection{Computational experiment setup}\label{sec:numsetup}

Throughout the experiments, we follow the Latent Twin training setup and learn the latent representation and the time-evolution surrogate jointly. In the numerical sections, we use $x$ for the finite-dimensional state available to the learning algorithm, while $u$ denotes an abstract state or underlying continuous field. Given trajectories of a dynamical system, we form training pairs $\left\{(x_s^j,s^j),(x_t^j,t^j) \right\}_{j=1}^J$, where $x_s^j$ and $x_t^j$ are states from the same trajectory evaluated at two time points $s^j,t^j\in[T_0,T_{\mathrm{end}}]$. In the PDE examples, these states correspond to finite-dimensional projections of the underlying solution, e.g., $x_s^j = R_N(u(s^j))$ and $x_t^j = R_N(u(t^j))$. The time pairs are sampled uniformly from the interval of interest, though other sampling strategies may also be used.

For all ODE Latent Twin examples considered in this section, trajectories are generated using \texttt{scipy.integrate.odeint} with $n_t=20000$ uniformly spaced time steps and tolerances $\mathrm{atol}=10^{-12}$ and $\mathrm{rtol}=10^{-10}$. From the resulting trajectories, we construct $J=2^{15}$ state pairs, of which $0.8J$ are used for training and $0.2J$ for testing. Unless stated otherwise, all models are trained with batch size $256$. For the PDE examples, trajectories are generated using the spectral solver \texttt{Dedalus} \cite{burns2020dedalus} on a periodic domain with randomized smooth initial conditions; details are given in \Cref{sec:PDEs}.

The training objective combines an autoencoder reconstruction loss, a prediction loss, and, depending on the numerical setup, additional conformity losses that encode the desired physical structure by matching the decoder-induced pullback and/or enforcing conformity of the latent evolution. We write the total empirical loss as
\begin{equation}\label{eq:loss}
    \mathcal L_{\mathrm{total}}(\theta)
        = \lambda_{\mathrm{auto}}\mathcal L_{\mathrm{auto}}(\theta)
        + \lambda_{\mathrm{pred}}\mathcal L_{\mathrm{pred}}(\theta)
        + \mathcal L_{\mathrm{conf}}(\theta) ,
\end{equation}
where the first two terms are, in our numerical experiments, chosen as squared $\ell^2$ losses,
\[
    \mathcal L_{\mathrm{auto}}(\theta)
        = \tfrac{1}{2J}\sum_{j=1}^J \left\|     d_\theta(e_\theta(x_s^j))-x_s^j \right\|^2
        + \tfrac{1}{2J}\sum_{j=1}^J     \left\|     d_\theta(e_\theta(x_t^j))-x_t^j  \right\|^2,
\]
and
\[
    \mathcal L_{\mathrm{pred}}(\theta)
        = \tfrac{1}{J}\sum_{j=1}^J \left\|     d_\theta\!\left(m_{\theta,s^j\to t^j}(e_\theta(x_s^j))\right)-x_t^j \right\|^2 .
\]
Here, $\theta$ denotes the trainable parameters of the encoder, decoder, and latent flow. Since $\mathcal L_{\mathrm{pred}}$ measures an error along the learned trajectory map, other norms may also be natural. In particular, an $\ell^1$ loss can provide a more directly interpretable accumulated trajectory discrepancy. For simplicity and consistency across the experiments, however, we use squared $\ell^2$ losses in the numerical investigations below.

The following losses operationalize the two mechanisms introduced in \Cref{sec:pclflow}: matching the decoder-induced pullback and enforcing latent conformity along the learned flow. The conformity contribution depends on the physical structure of the underlying system and may contain one or several terms. We distinguish between pullback conformity, which aligns a latent functional $\mathcal C_Z$ with the decoder-induced physical quantity $\mathcal C\circ d_\theta$, and latent conformity, which enforces preservation or dissipation of $\mathcal C_Z$ along the latent flow.

Accordingly, we write
\begin{equation}
    \mathcal L_{\mathrm{conf}}(\theta) =\lambda_{\mathrm{pullback}}\mathcal L_{\mathrm{pullback}}(\theta) + \lambda_{\mathrm{latent}}\mathcal L_{\mathrm{latent}}(\theta),
    \label{eq:conformity-loss}
\end{equation}
with $\mathcal L_{\mathrm{latent}}$ replaced by its dissipative variant when appropriate. The relevant terms and weights are chosen separately for each experiment. We set $\lambda_{\mathrm{auto}}=\lambda_{\mathrm{pred}}=1$ throughout. When the desired structure in the latent space can be enforced exactly, the corresponding conformity term is replaced by a hard architectural constraint on the latent flow $m_{s\to t}$.

The pullback term is used when the relation between the latent functional and the decoded physical quantity is learned rather than exact. In its simplest form, we match the latent functional to the decoder-induced physical quantity on encoded source states,
\begin{equation}\label{eq:pullback-loss}
    \mathcal L_{\mathrm{pullback}}(\theta) = \tfrac{1}{J} \sum_{j=1}^J \left\| \mathcal C_Z(e_\theta(x_s^j)) - \mathcal C(d_\theta(e_\theta(x_s^j))) \right\|^2 .
\end{equation}
This encourages $\mathcal C_Z$ to approximate the decoder-induced pullback $\mathcal C\circ d_\theta$ on the encoded data manifold. In some experiments, we also apply the same matching term to additional latent states, such as $e_\theta(x_t^j)$ or $m_{\theta,s^j\to t^j}(e_\theta(x_s^j))$, to align the pullback along states visited by the learned dynamics.

For equality-type constraints, latent conformity may be imposed through
\begin{equation} \label{eq:latent-conformity-loss}
    \mathcal L_{\mathrm{latent}}(\theta) = \tfrac{1}{J} \sum_{j=1}^J \left\| \mathcal C_Z\left( m_{\theta,s^j\to t^j}(e_\theta(x_s^j)) \right) - \mathcal C_Z(e_\theta(x_s^j)) \right\|^2.
\end{equation}
For dissipative constraints, this may be replaced by the one-sided penalty
\begin{equation}\label{eq:latent-conformity-loss-ineq}
    \mathcal L_{\mathrm{latent}}(\theta) = \tfrac{1}{J} \sum_{j=1}^J \left\| \left[ \mathcal C_Z\left( m_{\theta,s^j\to t^j}(e_\theta(x_s^j)) \right) - \mathcal C_Z(e_\theta(x_s^j)) \right]_+ \right\|^2 .
\end{equation}
Thus, $\mathcal L_{\mathrm{pullback}}$ connects $\mathcal C_Z$ to the decoded physical quantity, whereas $\mathcal L_{\mathrm{latent}}$ controls how $\mathcal C_Z$ evolves under the latent flow.

This general setup is intentionally kept simple. Additional terms, such as semigroup-consistency losses, regularization of the encoder and decoder, or state-space penalties on decoded physical quantities, can be incorporated naturally. In the present numerical section, however, we focus on the minimal modifications needed to demonstrate the effect of physics-conforming latent dynamics. Remaining experiment-specific architectures, data generation procedures, constraint choices, and evaluation metrics are described in the corresponding subsections. To support reproducibility and hands-on engagement with the proposed framework, we provide tutorial-style demonstration code on \href{https://github.com/matthiaschung/physics-conforming-latent-twins}{GitHub}, with additional code to be made available upon acceptance.

\subsection{ODE benchmarks} \label{sec:ODEs}

We consider four ODE benchmarks that progress from linear conservative dynamics to nonlinear, Hamiltonian, and dissipative systems, thereby illustrating different strategies for imposing latent constraints. In \Cref{sec:ROM}, we study a linear harmonic oscillator and compare the resulting Latent Twin with POD and DMD baselines. In \Cref{sec:sir}, we turn to a nonlinear multigroup SIR model, where total-population conservation is encoded through a softly enforced linear latent constraint. In \Cref{sec:pendulum}, we consider the undamped pendulum, using a hard Hamiltonian architectural constraint to enforce a quadratic latent invariant, together with an additional comparison against a PINN baseline. Finally, \Cref{sec:dissipative} studies the Rosenbrock gradient flow as a dissipative benchmark, in which monotone energy decay provides the relevant physics-conforming structure.

\subsubsection{Harmonic oscillator}\label{sec:ROM}

We train our Physics-conforming Latent Twins on stable linear ODE systems of the form
\begin{equation}\label{eq:linear-ode}
    x' = Mx
\end{equation}
where $M$ has purely imaginary eigenvalues, so the system exhibits energy-preserving oscillations. We consider a high-dimensional system with $n_x=100$, where the dynamics admit a low-rank representation $M=U^\top A U$, with $U \in \mathbb{R}^{10 \times 100}$ having orthonormal rows and $A \in \mathbb{R}^{10 \times 10}$ skew-symmetric. The initial condition is drawn randomly as $x_0 \sim \mathcal{U}([0,1]^{n_x})$ and trajectories are integrated over $t \in [0, 10]$.

In parallel to the Proper Orthogonal Decomposition (POD) convention, we choose a linear projection encoder $P \in \mathbb{R}^{10 \times 100}$ with orthonormal rows and decoder $P^\top$, yielding the projected latent twin model
\begin{equation}\label{eq:projected-latent-twin}
    \Psi(t,s,x_s) = P^\top\exp((t-s)W)P x_s.
\end{equation}

To promote the energy-preserving structure, we regularize the learned latent operator toward skew-symmetry,
\begin{equation} \label{eq:skew-symmetric-loss}
    \mathcal{L}_{\mathrm{latent}} = \|W+W^\top\|_{\mathrm F}^2, \qquad \lambda_{\mathrm{latent}} = 10^{-3}.
\end{equation}
Equivalently, this is an operator-level latent-conformity penalty for the quadratic latent invariant $\mathcal C_Z(z)=\|z\|_2^2$. The Physics-conforming Latent Twin is trained for $200$ epochs using Adam with fixed learning rate $10^{-3}$.

Since the linear reconstruction map already yields a quadratic decoder-induced pullback, we set $\lambda_{\mathrm{pullback}}=0$ and enforce the compatible latent invariant directly through the skew-symmetry constraint described in \Cref{prop:affine_linear_exact_constraints}.

We compare the Physics-conforming Latent Twins against two classical reduced-order modeling methods. POD projects the dynamics onto the leading $r=10$ empirical modes extracted from the training trajectories and integrates the projected system forward in time. In this particular setting, the POD solution serves as an optimal achievable projection benchmark for models of the form \Cref{eq:projected-latent-twin}. Dynamic Mode Decomposition (DMD) identifies a best-fit linear operator directly from snapshot pairs in the original state space and uses it for prediction. Both baselines use the same training data as the Physics-conforming Latent Twins (trajectory before training pair selection).

The results are presented in \Cref{fig:linear-ODE}. POD provides a favorable benchmark because it directly projects onto the dominant invariant subspace of the dynamics and therefore achieves near machine-precision error ($\approx 10^{-7}$). While the error of DMD grows to approximately $10^{-2}$ over the prediction horizon, the Physics-conforming Latent Twin maintains stable errors around $10^{-3}$. This stability indicates that the learned latent flow captures the long-time oscillatory structure of the system from paired state observations, without requiring explicit access to the governing matrix.

\begin{figure}
    \centering
    \includegraphics[width=0.7\linewidth]{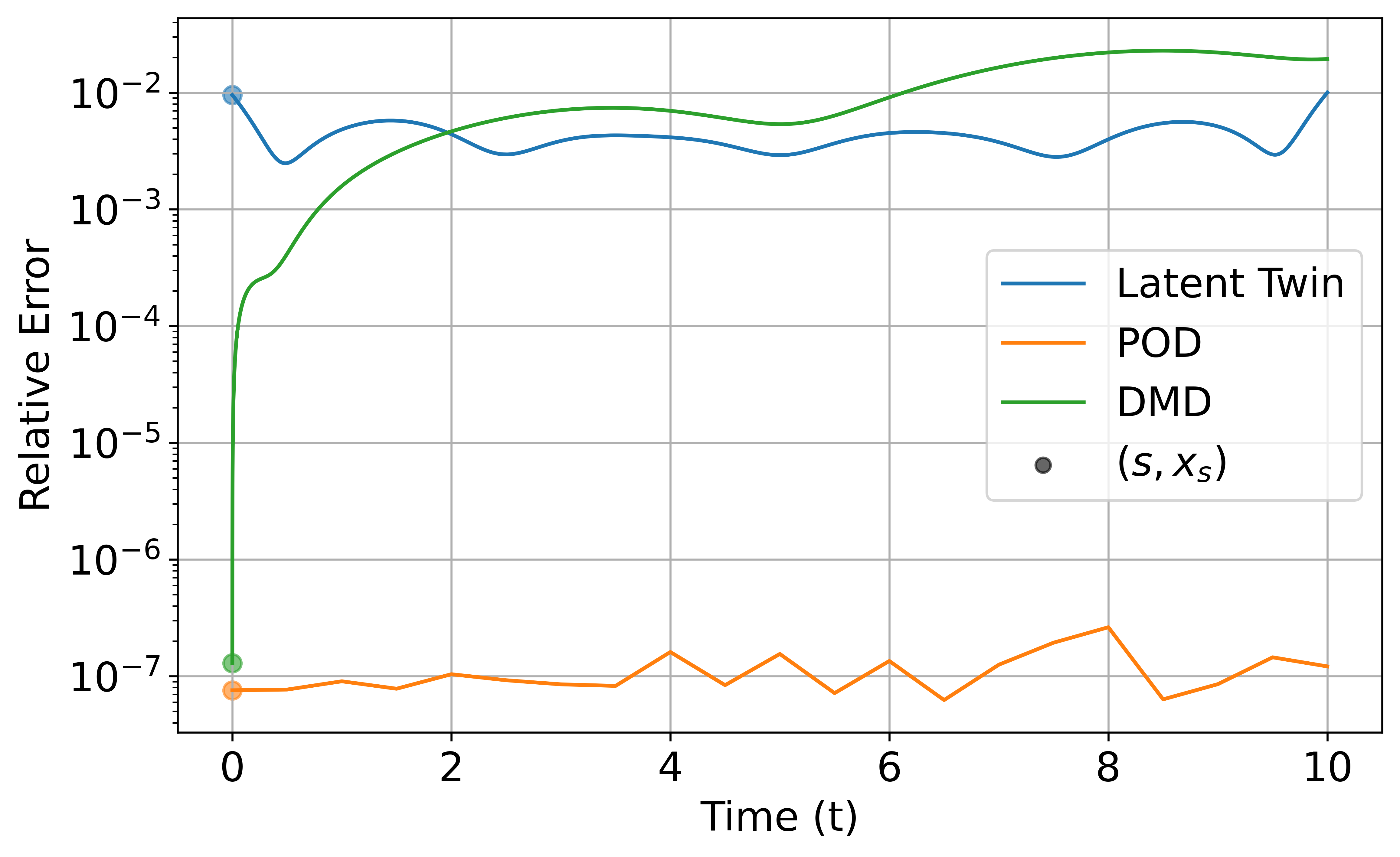}
    \caption{Comparison of the Physics-conforming Latent Twin, POD, and DMD on the stable linear ODE $x'=Mx$ with ambient dimension $n_x=100$ and intrinsic dimension $r=10$. POD provides an optimal projection benchmark in this constructed low-rank linear setting.}
    \label{fig:linear-ODE}
\end{figure}

\subsubsection{SIR}
\label{sec:sir}

We consider a multigroup SIR (Susceptible-Infected-Recovered) system with $n=20$ coupled groups, giving a state dimension of $60$ \cite{hethcote2000mathematics}. Each group evolves according to
\[
    S_i' = -\beta_i S_i \lambda_i, \qquad
    I_i' = \beta_i S_i \lambda_i - \gamma_i I_i, \qquad
    R_i' = \gamma_i I_i,
\]
where $\lambda_i = \sum_j M_{ij} I_j$ is the force of infection by the mixing matrix $M \in \mathbb{R}^{n\times n}$. The mixing matrix is constructed with rank $r = 3$, introducing low-rank coupling between groups. Transmission and recovery rates are drawn independently as $\beta_i \sim \mathcal{U}(0.15, 0.45)$ and $\gamma_i \sim \mathcal{U}(0.05, 0.25)$. Initial conditions are sampled from $S_i(0) \sim \mathcal{U}(0.7, 0.99)$, $I_i(0) \sim \mathcal{U}(0.005, 0.02)$, and $R_i(0) = 1 - S_i(0) - I_i(0)$. Trajectories are integrated over $t \in [0, 60]$.

The SIR dynamics evolve on a nonlinear manifold, so a linear autoencoder as utilized in \Cref{sec:ROM}, is insufficient for accurate long-horizon prediction as shown in \Cref{fig:linear_sir}. We therefore adopt a nonlinear autoencoder. The encoder and decoder are fully connected multilayer perceptrons with $\tanh$ activations. The encoder uses the width schedule $60 \to 16 \rightarrow 8 \rightarrow n_z$  and the decoder mirrors this architecture. We set $n_z=3$. The latent dynamics are modeled by the linear exponential flow $z_t = \exp((t-s)W)z_s$. Thus, the nonlinear autoencoder provides the expressive embedding, while the latent evolution remains linear and structure-constrained.

The SIR system preserves total population, $\frac{\mathrm d}{\mathrm dt} \sum_{i=1}^n (S_i+I_i+R_i)=0,$ which corresponds to a linear conservation law in the state space. Following \Cref{tab:latent_constraint_classes}, we encode this structure through the pullback and latent conformity losses in \Cref{eq:pullback-loss,eq:latent-conformity-loss}. We use
\[
    \mathcal C(x)=\mathbf 1^\top x, \qquad \mathcal C_Z(z)=\mathbf 1^\top z,
\]
and set $\lambda_{\mathrm{pullback}}=\lambda_{\mathrm{latent}}=1$. The pullback loss aligns the latent mass functional with the decoder-induced physical mass, while the latent conformity loss penalizes violations of this mass along the latent flow. Thus, the SIR experiment uses the general conformity loss in \Cref{eq:conformity-loss} with both pullback compatibility and latent conformity active.

As shown in \Cref{fig:nonlinear_sir}, the nonlinear autoencoder provides a representation in which the SIR dynamics are well captured by a linear exponential latent flow over the observed interval. The contrast with the linear autoencoder indicates that the main limitation is not the use of a linear latent flow, but the expressiveness of the embedding used to represent the nonlinear SIR solution manifold.

\begin{figure}
\centering
\begin{subfigure}{0.49\textwidth}
  \centering
  \includegraphics[width=\linewidth]{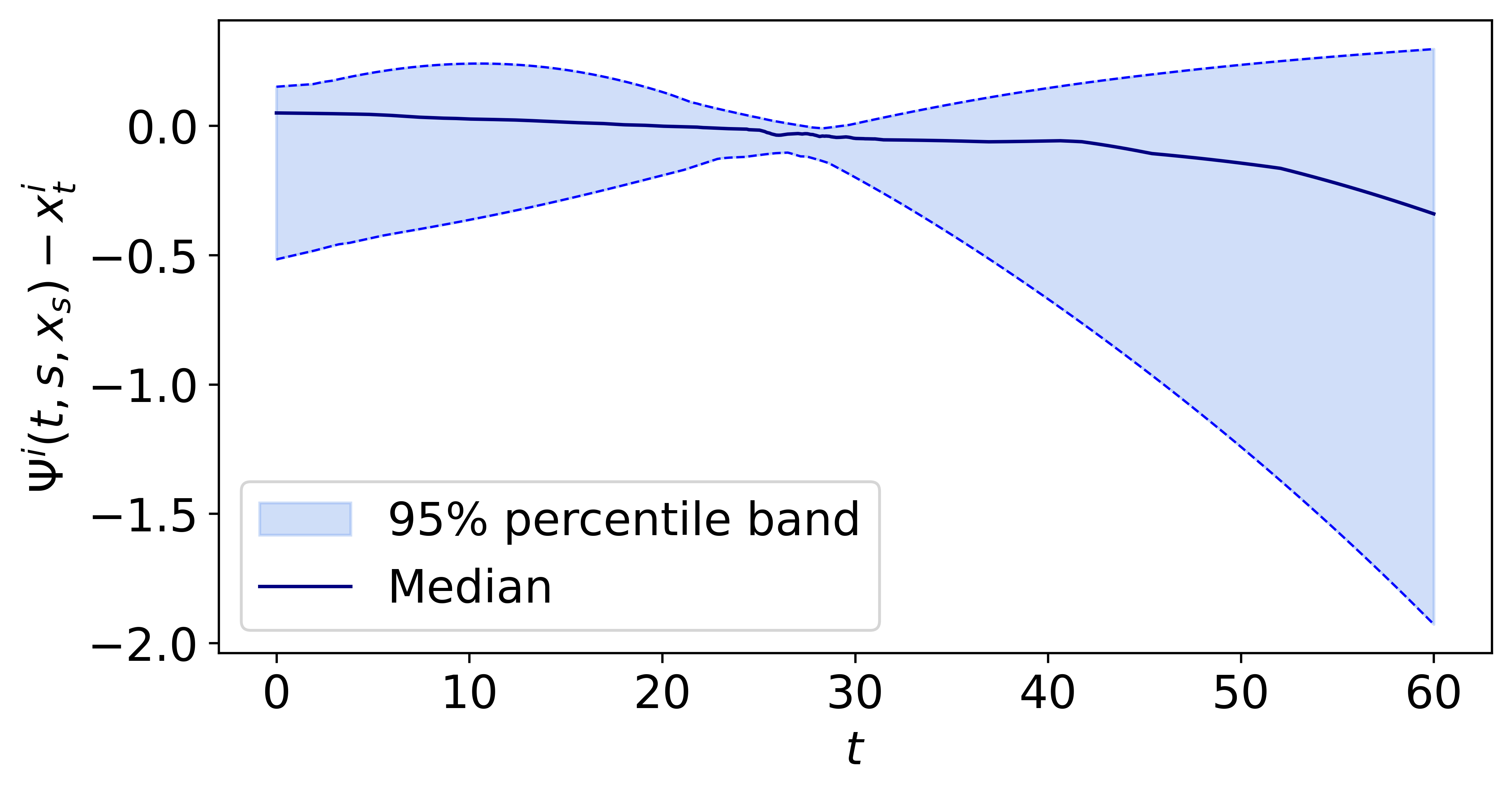}
  \caption{Linear autoencoder}
  \label{fig:linear_sir}
\end{subfigure}
\hfill
\begin{subfigure}{0.49\textwidth}
  \centering
  \includegraphics[width=\linewidth]{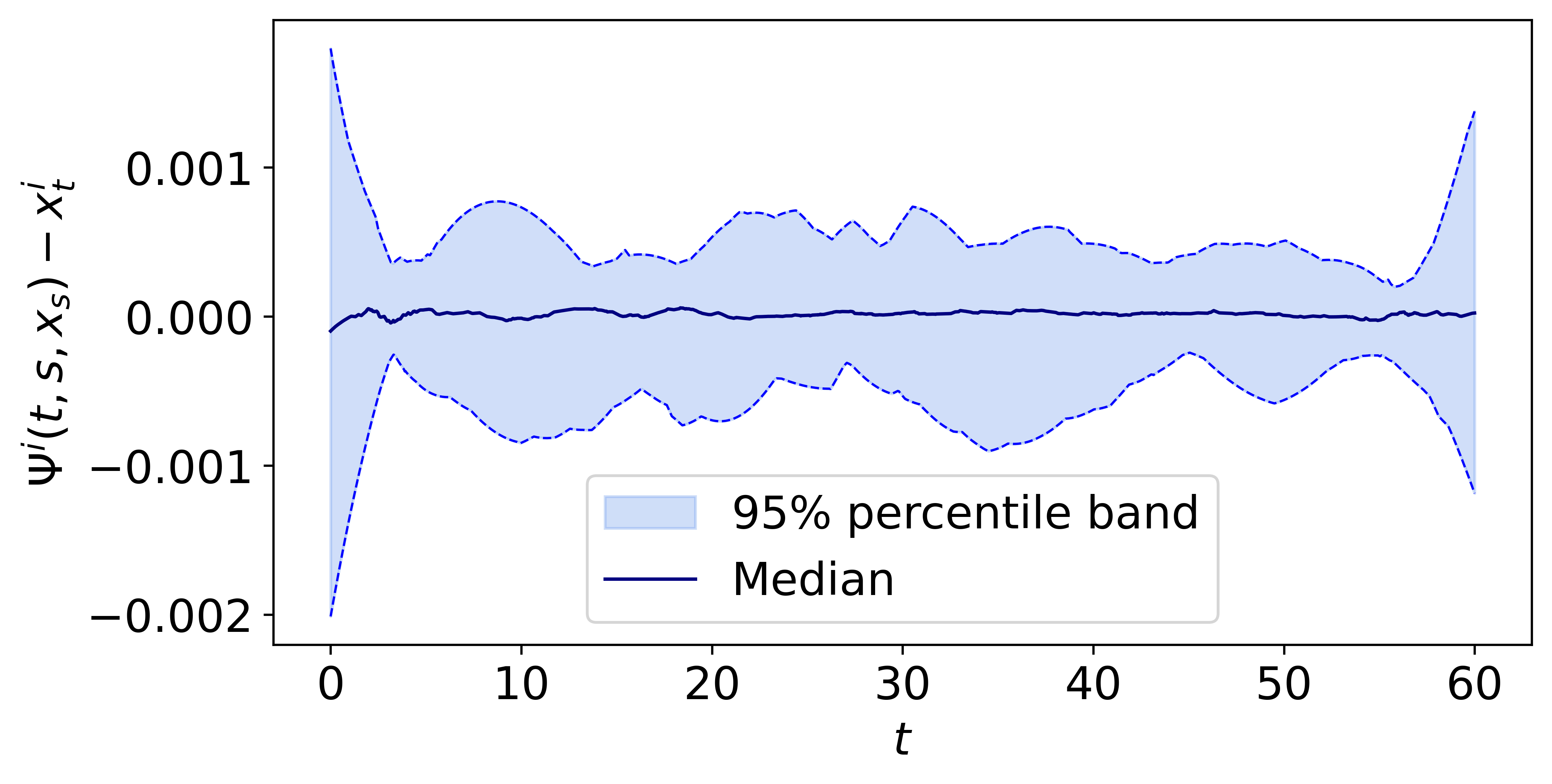}
  \caption{Nonlinear autoencoder}
  \label{fig:nonlinear_sir}
\end{subfigure}
\caption{95\% prediction-error bands for the SIR experiments.  (a) Linear autoencoder: $x_t = P^\top \exp((t-s)W)P x_s$.  (b) Nonlinear autoencoder: $x_t = (d \circ \exp((t-s)W)\circ e) (x_s)$. The nonlinear model reduces the error magnitude by approximately three orders of magnitude relative to the linear model. }
\label{fig:sir_models}
\end{figure}

We also compare against POD and DMD using the same reduced dimension $n_z=3$. POD is computed from the leading empirical modes and uses the projected right-hand side, making it equation-aware \cite{rowley2004model}. DMD instead identifies a best-fit linear evolution operator from snapshot data. In contrast, the Latent Twin learns both the nonlinear representation and the latent evolution directly from paired state observations.

The comparison results are shown in \Cref{fig:sir_compare}. POD is equation-aware, since it uses the projected right-hand side, while DMD is fit directly from the full snapshot trajectory. The Latent Twin nevertheless outperforms both baselines. This improvement reflects the advantage of learning the representation and latent evolution jointly: the nonlinear encoder--decoder captures the curved SIR solution manifold, while the physics-conforming latent operator exploits the low-rank coupling structure without requiring explicit knowledge of the governing equations.

\begin{figure}
    \centering
    \includegraphics[width=0.7\linewidth]{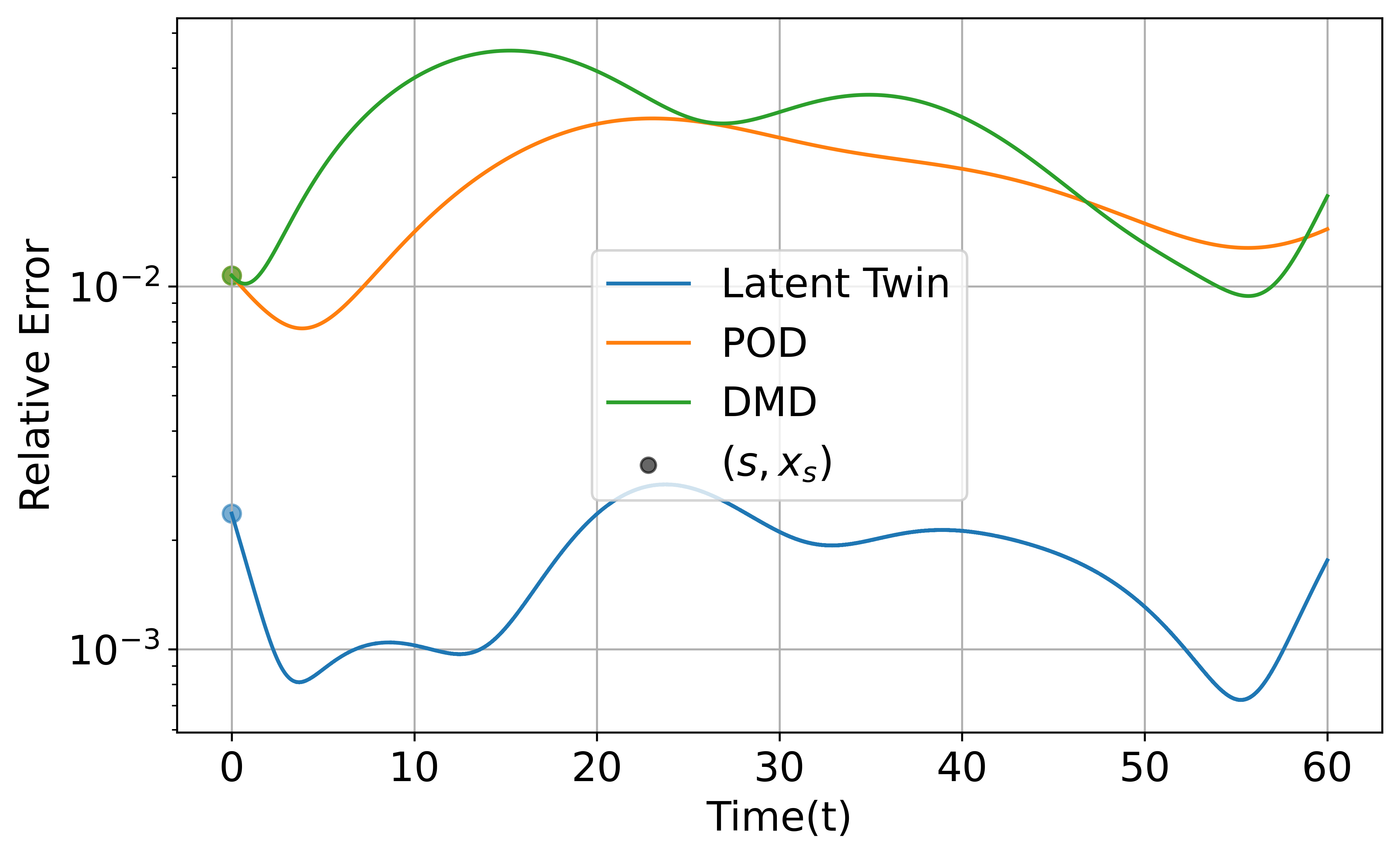}
    \caption{Relative prediction errors for the Physics-conforming Latent Twin, POD, and DMD on the multigroup SIR system with $n_x=60$ and reduced dimension $n_z=3$. The Latent Twin achieves lower error by combining a nonlinear state representation with pullback and latent conformity regularization that encode the population conservation structure.}
    \label{fig:sir_compare}
\end{figure}

\subsubsection{Undamped pendulum}\label{sec:pendulum}

We next consider the undamped pendulum
\[
    \theta'=\omega,\qquad
    \omega'=-\frac{g}{l}\sin\theta,
\]
with $l=1$, $g=9.8$, initial state $(\theta,\omega)=(\pi/4,0)$, and training horizon $t\in[0,10]$. The pendulum is a conservative Hamiltonian system with nonlinear physical energy
\[
    H(\theta,\omega)
    =
    \frac12 \omega^2+\frac{g}{l}\bigl(1-\cos\theta\bigr).
\]
This makes the example qualitatively different from the SIR experiment: the conserved quantity is not a linear functional of the state, so representing it through a simple linear constraint on the latent generator would require the encoder and decoder to absorb the nonlinear geometry of the Hamiltonian. Instead, we impose a quadratic latent invariant as a tractable compatible surrogate for the decoder-induced pullback of the physical Hamiltonian.

Specifically, we use a linear exponential latent flow $z_t=\exp((t-s)W)z_s$
and enforce preservation of the latent quadratic energy
$ E_Z(z)=\frac12 z^\top Q_Z z$, with $Q_Z=Q_Z^\top\succ0$. A sufficient condition for exact preservation of $E_Z$ is $W^\top Q_Z+Q_ZW=0$.
We impose this condition as a hard architectural constraint by parameterizing
\[
    W=JQ_Z,\qquad
    J^\top=-J,\qquad
    Q_Z=B^\top B+\varepsilon I, 
\]
where $B$ is trainable and $\varepsilon=10^{-4}$. Then, along any latent trajectory,
\[
    \frac{\mathrm d}{\mathrm dt}E_Z(z(t)) =  z(t)^\top Q_Z W z(t) = z(t)^\top Q_Z J Q_Z z(t) = 0,
\]
because $Q_ZJQ_Z$ is skew-symmetric. Thus the latent quadratic energy is conserved exactly by construction, so no trajectory-level latent-conformity penalty is needed. Pullback matching may still be used to align the quadratic latent invariant with the decoded physical Hamiltonian.

We evaluate the pendulum example through representative trajectory predictions and an ablation study. The trajectory comparison contrasts an unconstrained Latent Twin, the Physics-conforming Latent Twin, and an energy-regularized physics-informed neural network (PINN). The ablation study then isolates how different mechanisms for imposing the Hamiltonian structure affect prediction error and invariant preservation, including pullback matching and latent conformity through the hard quadratic invariant.

The unconstrained Latent Twin uses the same nonlinear encoder--decoder architecture and exponential latent flow as the Physics-conforming model, but with an unconstrained latent generator $W$. All Latent Twin variants use latent dimension $n_z=6$ and a nonlinear encoder--decoder with hidden dimensions $(16,8)$. We generate $2^{15}$ randomly sampled state pairs $\{(x_s,s),(x_t,t)\}$ from the observed interval $[0,10]$ and use an $80/20$ train--test split. The models are trained jointly for $500$ epochs using Adam with learning rate $10^{-3}$, batch size $256$, and a step scheduler that halves the learning rate every $150$ epochs. We set $\lambda_{\mathrm{auto}}=\lambda_{\mathrm{pred}}=1$ and vary only the conformity terms and latent flow parameterization across the ablation. Each Latent Twin model has only $508$ trainable parameters. Importantly, the Latent Twin models are not given the pendulum differential equation; they use paired trajectory data and, when applicable, prescribed conformity terms.

As a reference physics-informed baseline, we train an energy-regularized PINN, following the PINN framework of \cite{raissi2019physics} and the conservation-law regularization perspective of \cite{baez2024guaranteeing}. The PINN represents the full trajectory directly as a map
\[
    \Psi_{\mathrm{PINN}}:t\mapsto \bigl(\theta_\eta(t),\omega_\eta(t)\bigr)
\]
using a fully connected network with four hidden layers, $64$ neurons per layer, and $\tanh$ activations. Its loss penalizes the residuals
\[
    r_1(t)=\partial_t\theta_\eta(t)-\omega_\eta(t),
    \qquad
    r_2(t)=\partial_t\omega_\eta(t)+\frac{g}{l}\sin(\theta_\eta(t)),
\]
at $500$ collocation points in $[0,10]$, together with a supervised data term on $20$ early-time samples from $[0,2]$ and an energy-conservation penalty enforcing $\frac{\mathrm d}{\mathrm dt}H(\theta_\eta(t),\omega_\eta(t))\approx0$. We use Adam with learning rate $10^{-3}$ for \num{100000} iterations, weight the data term by $100$, and weight the energy penalty by $100$. Time is normalized to $[0,1]$ during training. Since the PINN uses sparse supervised trajectory data but has access to the governing equation and energy functional, whereas the Latent Twins use paired state data without equation residuals, the comparison should be interpreted as a qualitative comparison of extrapolation behavior rather than a strict data-efficiency benchmark. The PINN has $12{,}738$ trainable parameters, compared with $508$ parameters for each Latent Twins model.

\Cref{fig:pendulum} shows the resulting predictions on the extended interval $[0,15]$. The unconstrained Latent Twin fits the observed interval reasonably well but accumulates phase and amplitude error after $t=10$, reflecting the fact that an unconstrained latent generator has no mechanism to enforce conservative long-time behavior. The energy-regularized PINN gives very accurate agreement on the observed interval; in fact, its in-distribution error is smaller than that of both the standard PINN and the Physics-conforming Latent Twins. However, its accuracy still deteriorates outside the observed interval. In contrast, the Physics-conforming Latent Twin maintains a comparable error level beyond $t=10$. The improvement is not due to access to the governing equation or an explicit Hamiltonian residual, but to the hard latent Hamiltonian structure: the learned flow is restricted to an energy-preserving class, which suppresses artificial growth or decay in the latent dynamics and improves long-time behavior after decoding.

This experiment highlights a central advantage of the Physics-conforming Latent Twins formulation. Rather than enforcing the known differential equation pointwise, as in a PINN, the method learns a solution operator in a latent space whose admissible dynamics already respect the relevant qualitative structure. For conservative systems, this distinction is especially important: preserving the geometric form of the dynamics can be more effective for extrapolation than merely reducing a residual loss over the observed interval.

\begin{figure}
\centering
\begin{tabular}{@{}c@{\hspace{1mm}}cc@{}}
\rotatebox{90}{\footnotesize \shortstack{Unconstrained\\ Latent Twins}} &
\includegraphics[width=0.43\textwidth, trim={1.0cm 1.0cm 0.0cm 0.0cm},clip]{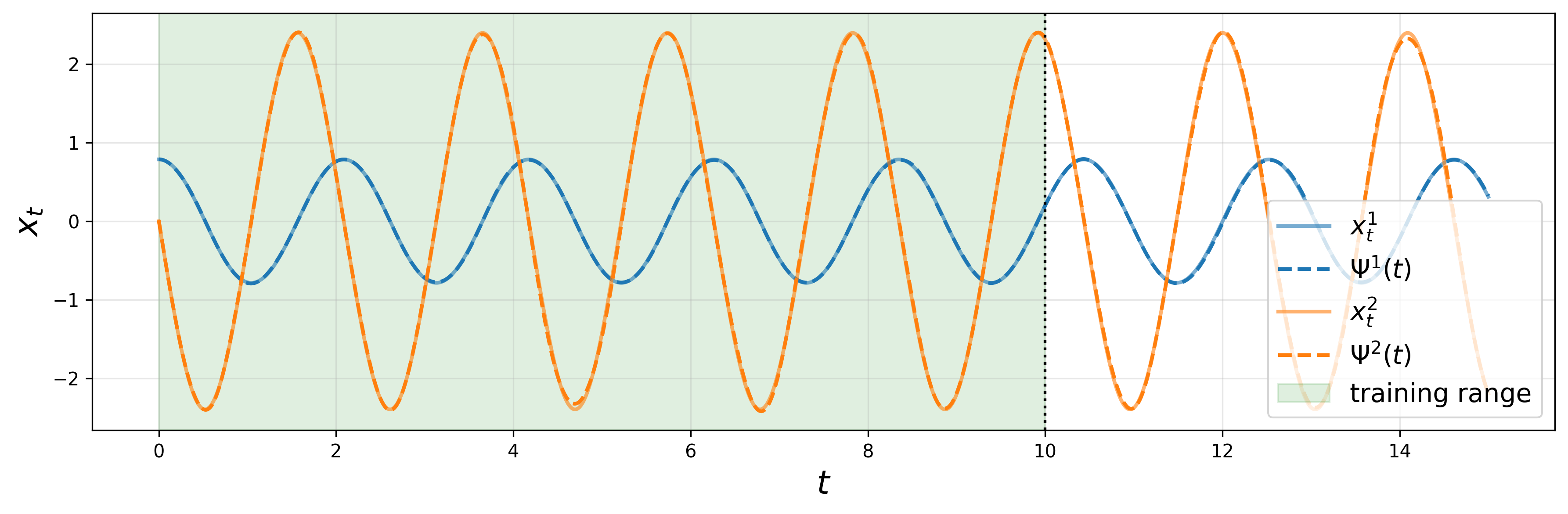} &
\includegraphics[width=0.43\textwidth, , trim={1.0cm 1.0cm 0.0cm 0.0cm},clip]{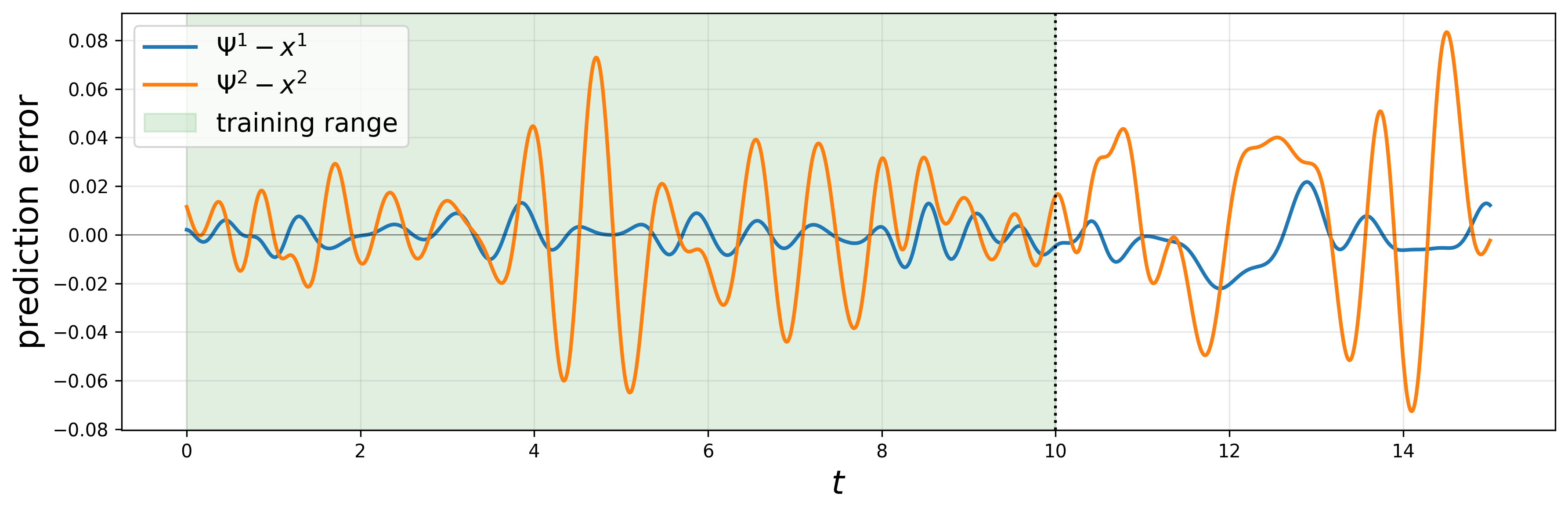} \\[0.5em]

\rotatebox{90}{\footnotesize \shortstack{Energy-\\regularized\\ PINN}} &
\includegraphics[width=0.43\textwidth, trim={1.0cm 1.0cm 0.0cm 0.0cm},clip]{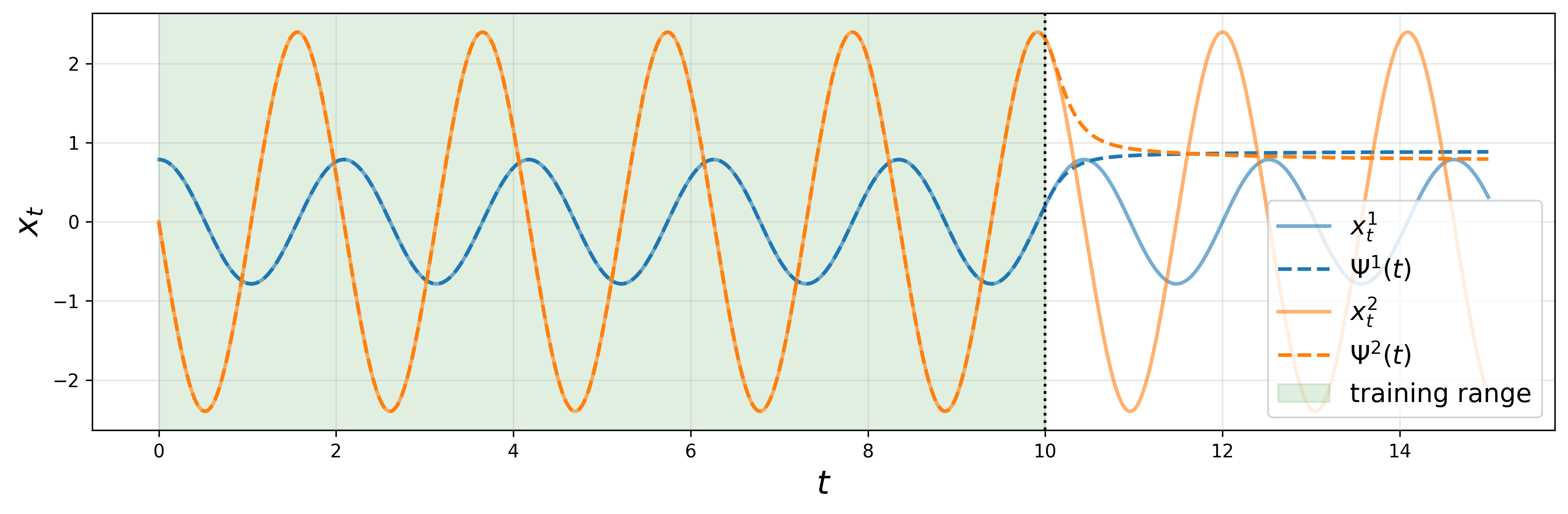} &
\includegraphics[width=0.43\textwidth, trim={1.0cm 1.0cm 0.0cm 0.0cm},clip]{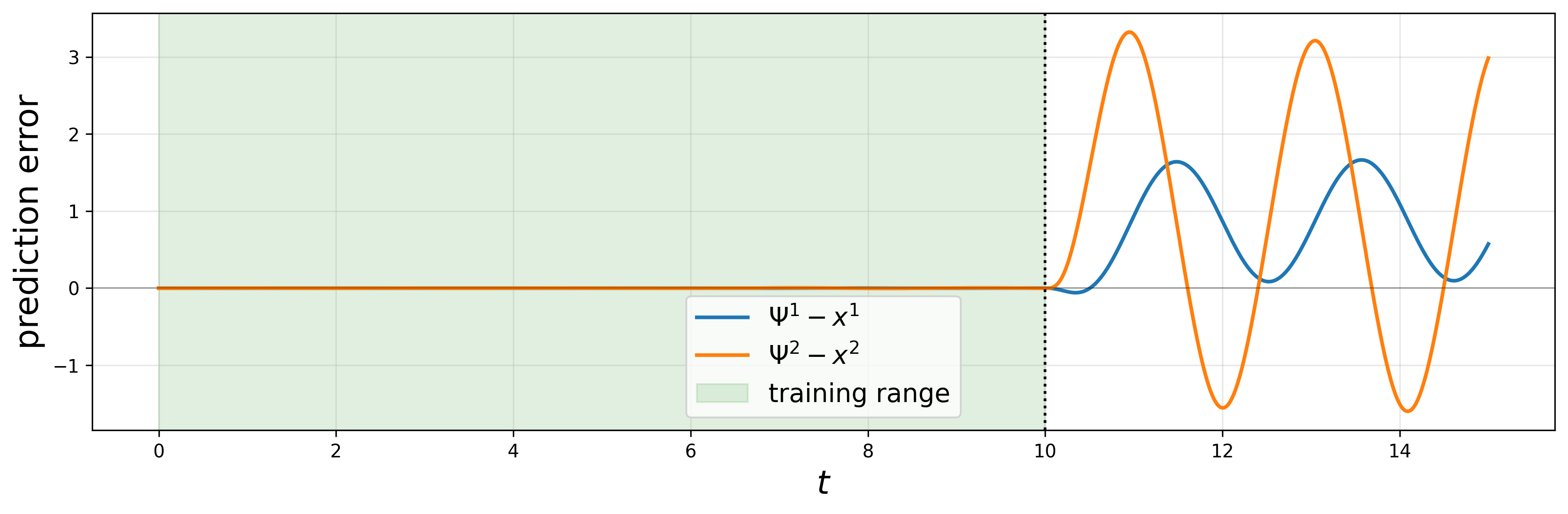} \\[0.5em]

\rotatebox{90}{\footnotesize \shortstack{Physics-\\conforming\\ Latent Twins}} &
\includegraphics[width=0.43\textwidth, , trim={1.0cm 1.0cm 0.0cm 0.0cm},clip]{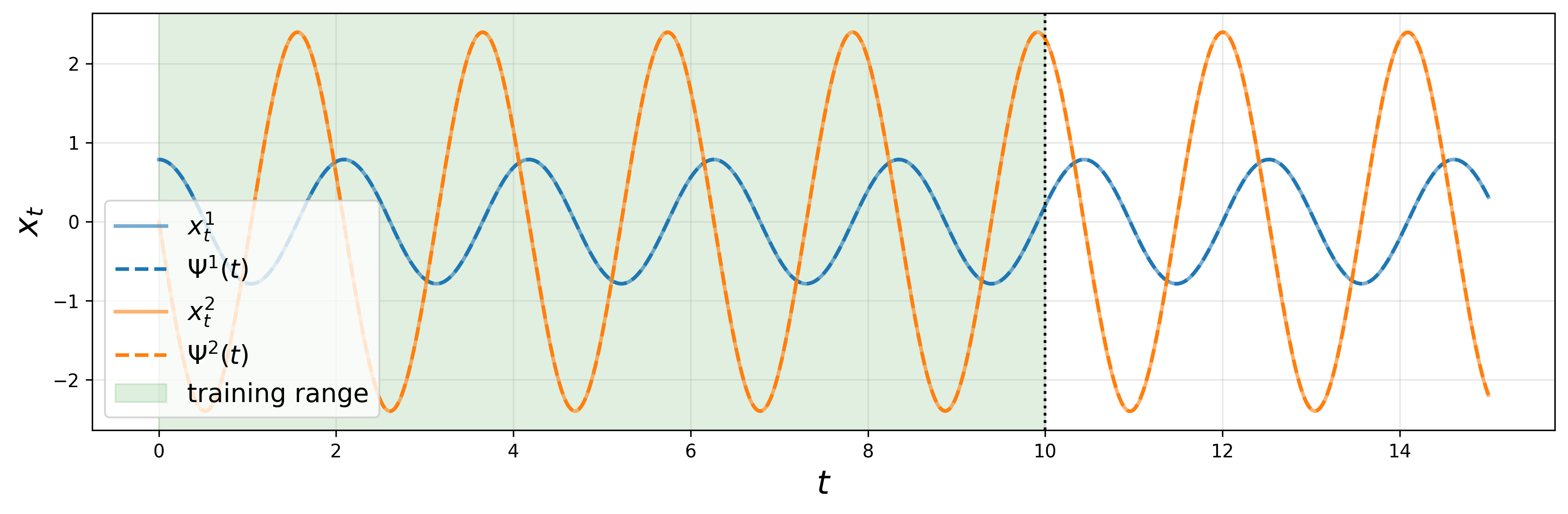} &
\includegraphics[width=0.43\textwidth, , trim={1.0cm 1.0cm 0.0cm 0.0cm},clip]{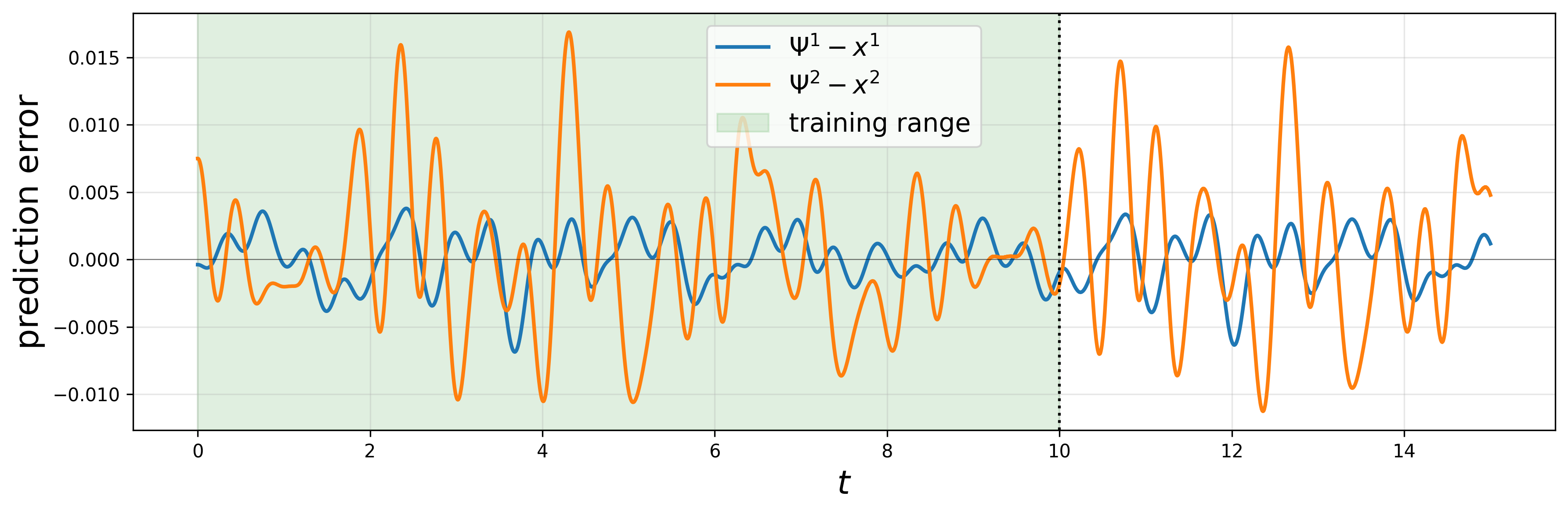}

\end{tabular}
\caption{Undamped pendulum predictions on an extended time interval. All models are trained using information from $t\in[0,10]$ and evaluated on $t\in[0,15]$. The left column shows the predicted state trajectory and the right column shows the prediction error. The unconstrained Latent Twin loses accuracy under extrapolation because the learned latent generator is not structurally restricted. The energy-regularized PINN uses both the governing ODE residual and an energy-conservation penalty, leading to excellent accuracy on the observed interval but degraded extrapolation beyond it. The Physics-conforming Latent Twin preserves a quadratic latent invariant by construction and maintains stable, accurate extrapolation beyond the training horizon.}
\label{fig:pendulum}
\end{figure}

The ablation study in \Cref{fig:compare_performance_and_invariant} shows that the way in which the Hamiltonian structure is imposed has a pronounced effect on extrapolation. The unconstrained Latent Twin performs reasonably on the observed interval, but loses accuracy beyond it. Pullback matching substantially reduces both prediction error and invariant defect by aligning the latent invariant with the decoder-induced physical Hamiltonian. Combining pullback matching with hard latent conformity performs best overall: the latent invariant is physically aligned through the decoder and is simultaneously preserved by the quadratic latent flow parameterization.

\begin{figure}
    \centering
    \includegraphics[width=0.99\linewidth]{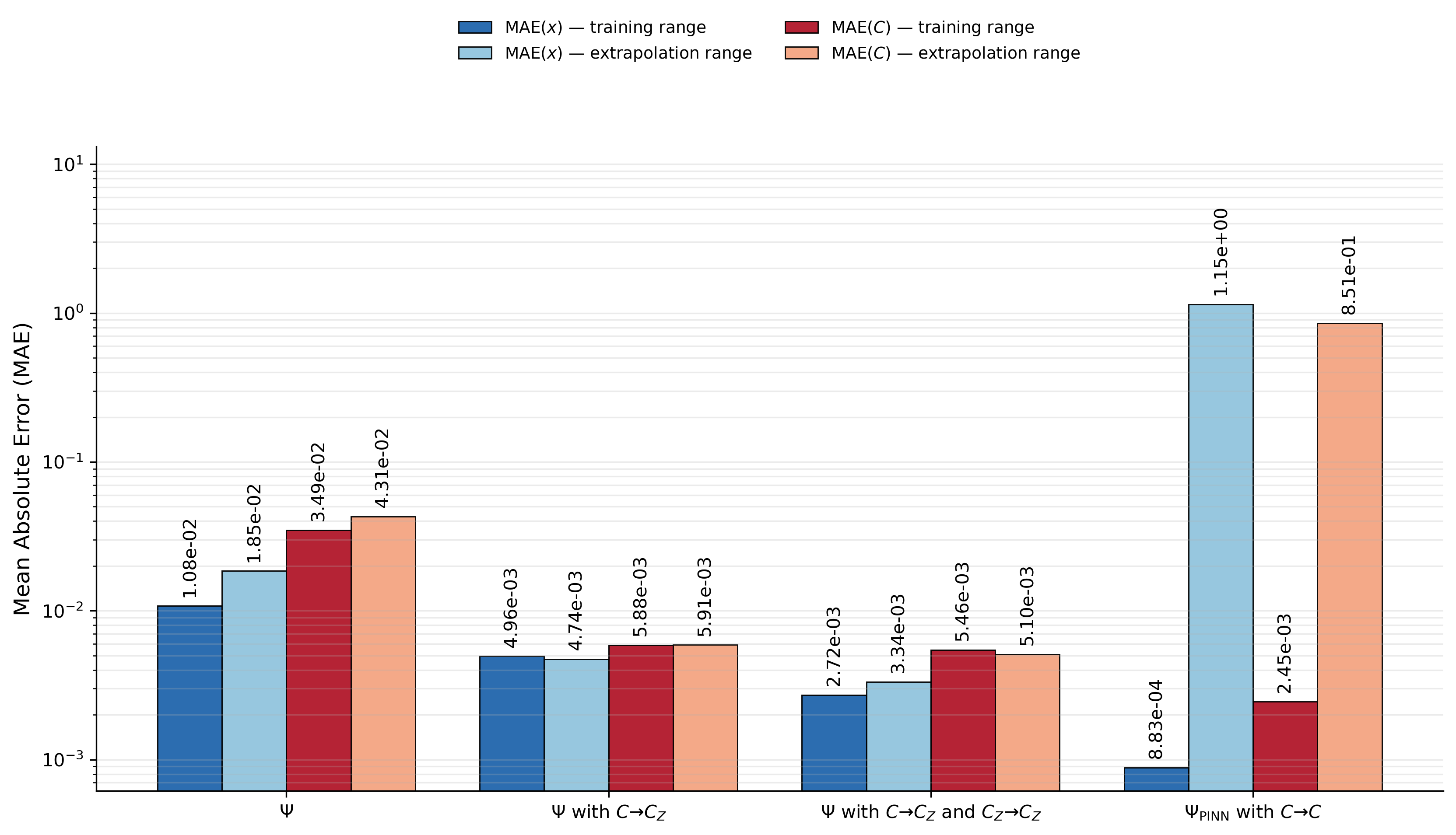}
    \footnotesize
    \caption{Ablation study for the undamped pendulum. Mean absolute prediction errors and Hamiltonian defects are reported on the observed interval $[0,10]$ and the extrapolation interval $[10,15]$. Pullback matching aligns the latent invariant with the decoder-induced Hamiltonian, while latent conformity enforces preservation by the latent flow. Their combination gives the most reliable extrapolation performance and small invariant defects.}
    \label{fig:compare_performance_and_invariant}
\end{figure}

The SIR and pendulum experiments illustrate complementary forms of physics conformity. In SIR, a linear conserved quantity is enforced through pullback and latent conformity losses. In the pendulum example, latent conformity alone already improves extrapolation, but the ablation study shows that the strongest results are obtained when the latent invariant is also aligned with the decoder-induced physical Hamiltonian. Thus, the experiments support the central distinction of the framework: compatibility with the decoded physical structure and consistent propagation by the latent flow.

\subsubsection{Gradient flow}\label{sec:dissipative}

As a dissipative benchmark, we consider gradient-flow dynamics associated with the optimization problem $ \min_{x\in\mathbb R^2} f(x)$, where $f(x_1,x_2)=(1-x_1)^2+100(x_2-x_1^2)^2$. The corresponding dynamics are
\[
    x'(t)=-\nabla f(x(t)),
\]
so that the objective is nonincreasing along exact trajectories: $\frac{\mathrm d}{\mathrm dt} f(x(t)) = -\|\nabla f(x(t))\|^2 \leq 0$. Thus, the Rosenbrock flow provides a simple nonlinear test case in which the learned surrogate should approximate the solution operator while respecting a Lyapunov-type dissipative structure.

Training trajectories are generated from uniformly sampled initial conditions
$x(0)\in[-1.5,1.5]^2$ over the time interval $[0,10]$. We use \num{4096} simulated trajectories, store \num{300} time points per trajectory, and sample \num{16} ordered pairs $s\leq t$ from each trajectory, giving \num{65536} source--target pairs $(s, x(s)), (t,x(t))$. To better resolve the fast initial transient, we use the reparametrized time grid $\tau = \left(t/t_{\max}\right)^{1/4} t_{\max}$, and sample training pairs with a mild bias toward early times. The data are split into $95\%$ training and $5\%$ testing samples and normalized componentwise by min--max scaling.

The Latent Twins use a fully connected autoencoder with encoder widths $2 \rightarrow 32 \rightarrow 32 \rightarrow 32 \rightarrow 16$, SiLU activations, and a mirrored decoder. We compare an unconstrained Latent Twin with exponential latent flow against a physics-conforming variant with the same latent flow but with pullback compatibility and latent dissipation. The physical dissipative quantity is the Rosenbrock objective, $\mathcal C(x)=f(x)$,
and we use the quadratic latent functional $\mathcal C_Z(z)=\|z\|^2$. The pullback loss with $\lambda_{\mathrm{pullback}}=1$ aligns $\mathcal C_Z(e(x))$ with the decoder-induced objective $f(d(e(x)))$. For a source latent state $z_s=e(x(s))$, the latent evolution is $z_t=m_{s\to t}(z_s)=\exp((t-s)W)z_s$, and the dissipative latent penalty is
\[
   \mathcal L_{\mathrm{latent}} = \left[ \|m_{s\to t}(z_s)\|^2-\|z_s\|^2 \right]_+^2, \qquad \lambda_{\mathrm{latent}}=1.
\]
All models are trained with Adam using batch size \num{1024}, learning rate $10^{-3}$, and \num{2000} epochs, with the learning rate halved every \num{500} epochs.

\Cref{fig:OptContour} shows the relative prediction error and a representative phase portrait for the test initial condition $x_s=(-0.5,-0.5)$. Both models track the trajectory fairly accurately on the training horizon, while the extrapolation region beyond $t=10$ separates their behavior. The unconstrained model can drift away from the Rosenbrock valley, whereas the dissipative model remains closer to the true gradient-flow trajectory.

The dissipative penalty also has the intended effect in latent space: after training, violations of the nonexpansivity condition
$\|m_{s\to t}(z_s)\|^2 \leq \|z_s\|^2$ are essentially eliminated for the physics-conforming model, while they persist for the unconstrained Latent Twin.

\begin{figure}
    \centering
    \includegraphics[width=0.48\linewidth]{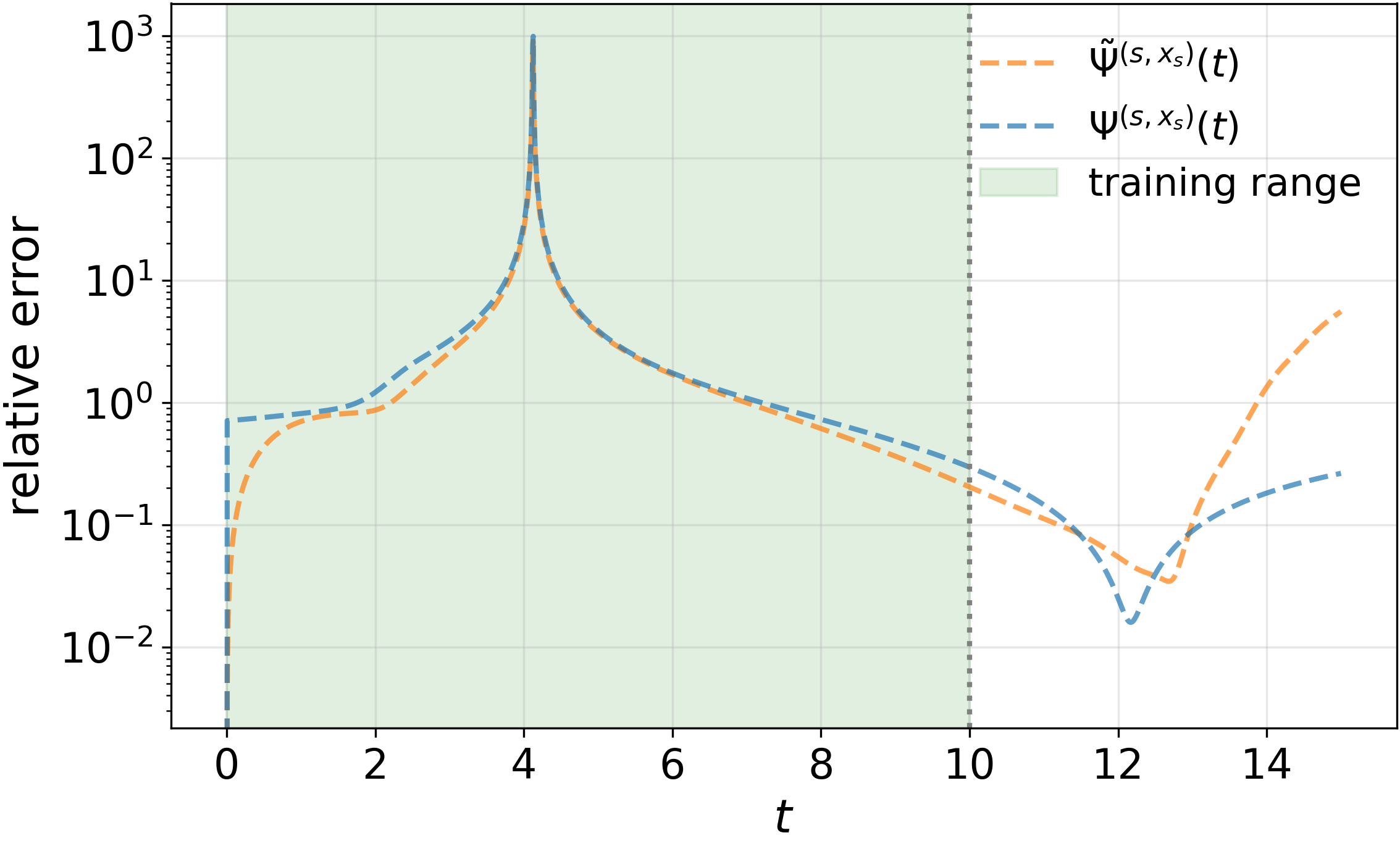}
    \includegraphics[width=0.48\linewidth]{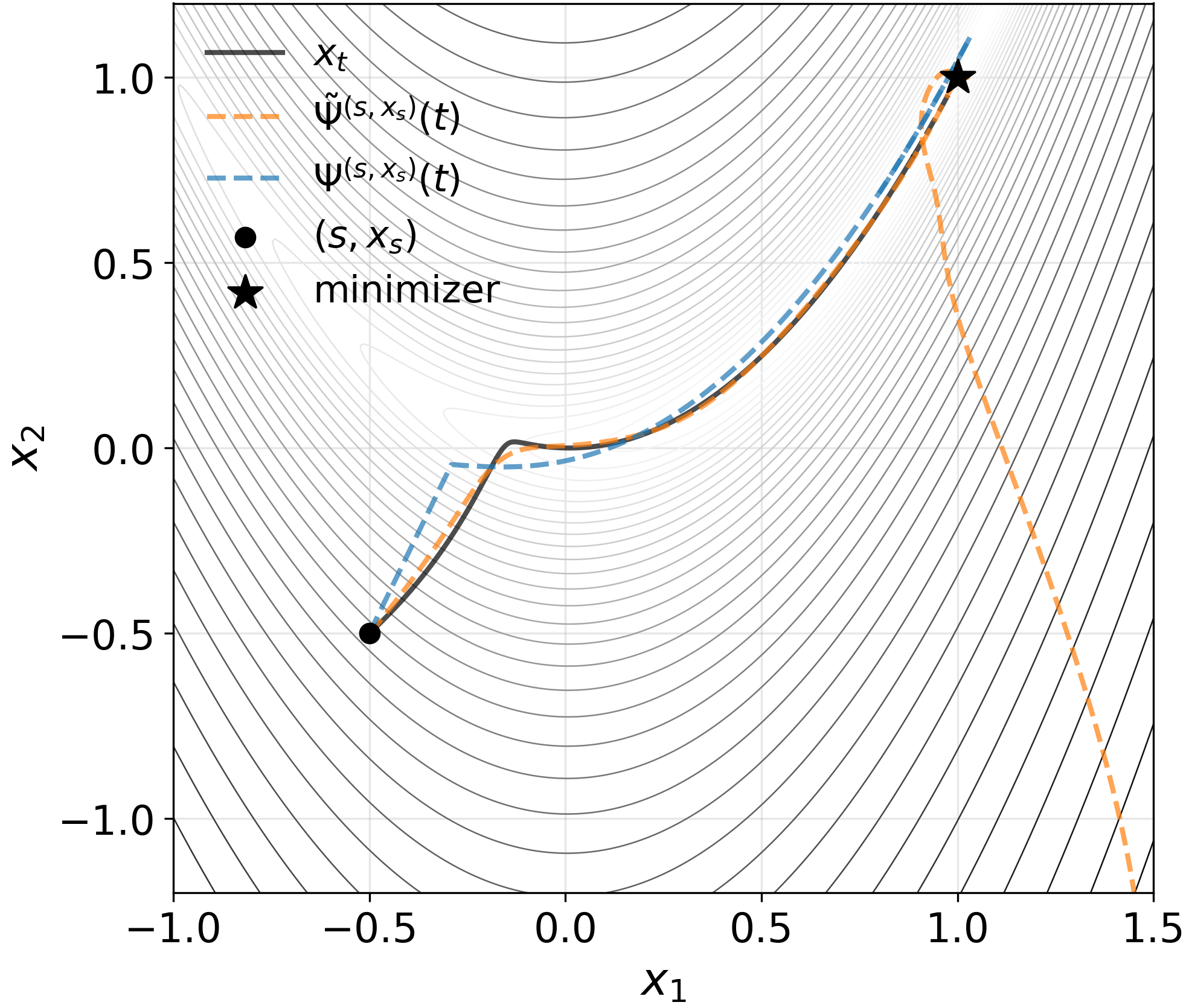}
    \caption{Gradient-flow Rosenbrock example. Left: relative prediction error for the unconstrained Latent Twin $\tilde\Psi(t,s,x_s)$ and the dissipative Latent Twin $\Psi(t,s,x_s)$, with the shaded region indicating the training horizon. Right: representative trajectory in the Rosenbrock landscape. Both models are accurate on the training range, while the dissipative model remains closer to the true trajectory under extrapolation.}
    \label{fig:OptContour}
\end{figure}

\subsection{PDEs}\label{sec:PDEs}

We extend the framework to two canonical PDE systems on the periodic domain $\Omega=[0,2\pi]^2$: the heat equation, which is dissipative, and the wave equation, which is conservative Hamiltonian. Both are simulated using the spectral solver \texttt{Dedalus} \cite{burns2020dedalus} with periodic boundary conditions. Initial conditions are random smooth periodic fields constructed from finite Fourier expansions with randomized amplitudes and phases. From each trajectory we extract snapshot pairs $(x(s), x(t))$ with $s \leq t$ drawn uniformly, forming training pairs as in \Cref{sec:numsetup}. The encoder and decoder are convolutional networks and the latent flow takes the exponential form $m_{s\to t}(z)=\exp((t-s)W)z$ in both experiments. Models are pretrained for $20$ epochs on the autoencoder loss and then trained jointly for $150$ epochs with Adam at learning rate $3\times10^{-4}$ and batch size $64$.

\subsubsection{Heat equation} \label{sec:heat}

We next consider the two-dimensional heat equation $\partial_t u = \kappa \Delta u$, which describes the diffusion of a scalar temperature field $u(x,y,t)$. On a periodic domain, the heat equation exhibits two complementary structural properties. First, the spatial mean $C_M[u](t) = \int_\Omega u(x,y,t)\,\mathrm{d}x\,\mathrm{d}y$ is conserved. Second, the $L^2$ energy $C_E[u](t) = \tfrac12 \int_\Omega u(x,y,t)^2\,\mathrm{d}x\,\mathrm{d}y$
is monotonically dissipated, since $
    \dfrac{\mathrm{d}}{\mathrm{d}t} C_E[u](t) = -\kappa \int_\Omega |\nabla u(x,y,t)|^2\, \mathrm{d}x\,\mathrm{d}y \leq 0. $
This example therefore falls into the mixed constraint class described in \Cref{tab:latent_constraint_classes}: a linear conserved quantity together with a dissipative quadratic functional.

For the latent model, we use latent dimension $n_z=32$ and a linear latent flow generated by an operator $W$. The physics-conforming parameterization decomposes
\[
    W = K + S,  \qquad  K = A-A^\top,  \qquad  S = -BB^\top - 0.001\, I,
\]
so that the symmetric part is negative definite and therefore induces latent energy decay. Conservation of the latent mass functional is enforced exactly by projecting the operator onto the constraint
subspace satisfying $W^\top c = 0$. Thus, the latent dynamics combine a conservative linear constraint with a dissipative quadratic structure, mirroring the physical behavior of the heat equation.

The results show that this constrained latent flow accurately captures the diffusive evolution. In \Cref{fig:heatRollout}, the predicted temperature fields remain visually indistinguishable from the reference solution over the full horizon. The model reproduces both the large-scale smoothing induced by diffusion and the gradual attenuation of high-frequency spatial features. The absolute errors remain small throughout the trajectory and are concentrated mainly in localized regions where the initial condition contains sharper spatial variations.

The reconstruction and prediction plots further indicate that the autoencoder learns a stable low-dimensional representation of the heat dynamics: reconstruction errors are small, while the evolution errors remain controlled even after repeated propagation in latent space. This is particularly important for the heat equation, where an unconstrained latent model could easily introduce spurious growth or drift in the mean temperature.

Finally, \Cref{fig:heatDiag} confirms that the learned dynamics respect the intended physical structure. The predicted total heat remains nearly constant along the horizon, while the latent energy decays monotonically, consistent with the dissipative character of the heat equation. Thus, the model does not only approximate the solution fields accurately; it also reproduces the qualitative conservation and dissipation laws that characterize the underlying PDE.

\begin{figure}
    \centering
    \includegraphics[width=0.9\linewidth]{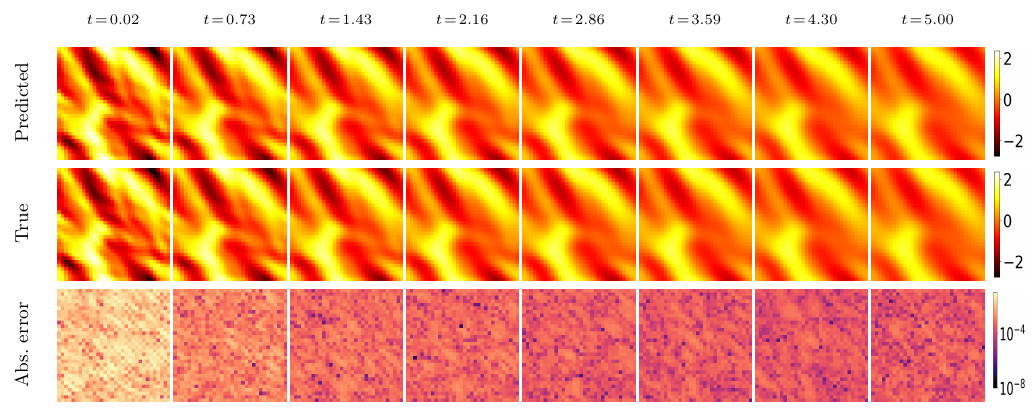}
    \caption{Heat equation predictions. Predicted temperature fields (row 1), true temperature fields (row 2), and absolute errors (row 3) at eight time snapshots over the horizon $t\in[0,5]$. The Physics-conforming Latent Twins accurately tracks the diffusive evolution, with small localized errors throughout the horizon.}

\label{fig:heatRollout}

\end{figure}

\begin{figure}
    \centering
    \includegraphics[width=0.7\linewidth]{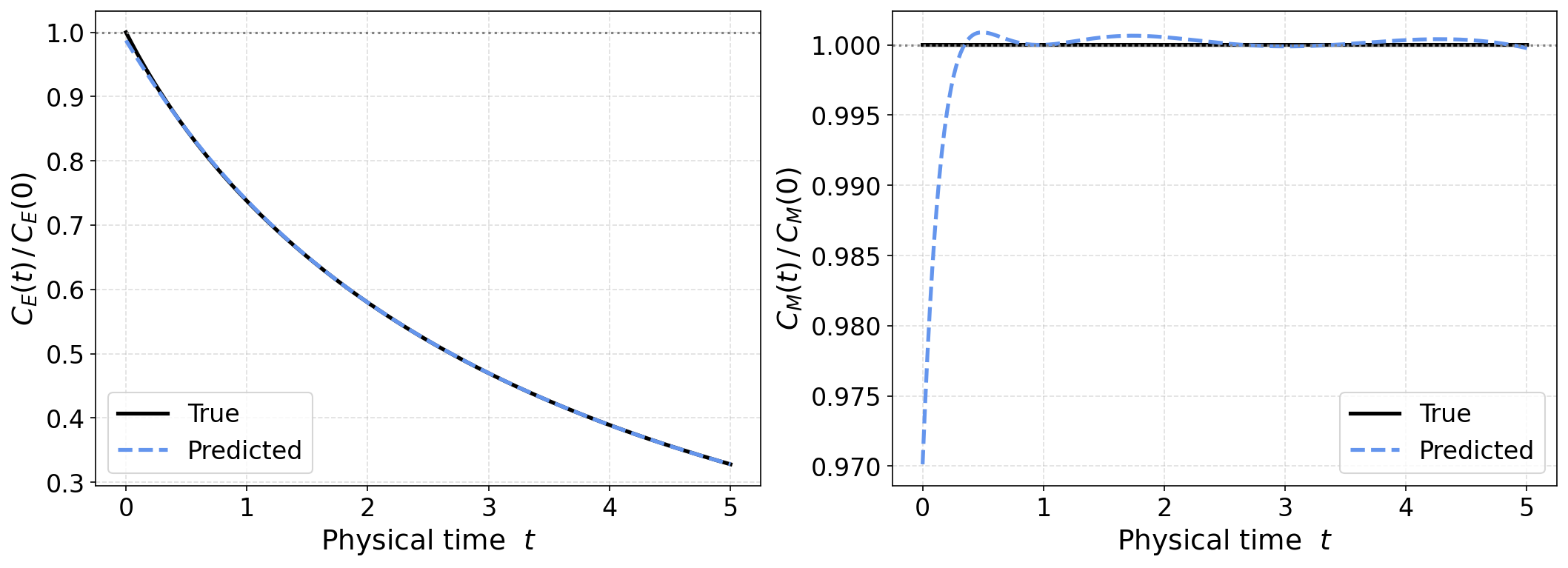}
    \caption{Conservation diagnostics for the heat Latent Twin. Top: $L^2$-energy $C_E[u](t)$ for the true trajectory and the surrogate prediction, showing monotone decay in both. Bottom: spatial mean $C_M[u](t)$, which the physics-conforming
    model preserves to near-machine precision.}
    \label{fig:heatDiag}
\end{figure}

\subsubsection{Wave equation} \label{sec:wave}

As a conservative PDE benchmark, we consider the wave equation in first-order form,
\[
    \partial_t q = p,
    \qquad
    \partial_t p = c^2 \Delta q ,
\]
where $q$ denotes the displacement field and $p$ its velocity. In contrast to the heat equation, the wave equation is nondissipative: oscillatory spatial structures propagate through the domain rather than being smoothed out. On a periodic domain, the continuous dynamics conserve the Hamiltonian
\[
    C_H[q,p] =  \tfrac12
    \int_\Omega
    \left(
        p^2 + c^2 |\nabla q|^2
    \right)
    \,\mathrm{d}x\,\mathrm{d}y .
\]
Thus, the relevant physical structure is conservation of a quadratic energy functional.

We encode this conservative structure by using a quadratic latent invariant. The latent dimension is set to $n_z=64$ to represent the two-field state $(q,p)$. The latent dynamics are generated by a linear operator $W$ parameterized as $W = A-A^\top$, so that $W^\top+W=0$. Consequently, the latent flow preserves the quadratic latent energy $C_z(t)= \tfrac12 \|z(t)\|^2$ exactly, i.e., $\frac{\mathrm d}{\mathrm dt}\, \tfrac12 \|z(t)\|^2 = z(t)^\top W z(t) = 0$.
This hard architectural constraint prevents artificial latent damping or growth and mirrors the conservative character of the wave equation.

\begin{wrapfigure}{r}{0.5\linewidth}
    \centering
    \includegraphics[width=\linewidth, trim={0 0 0 0.75cm},clip]{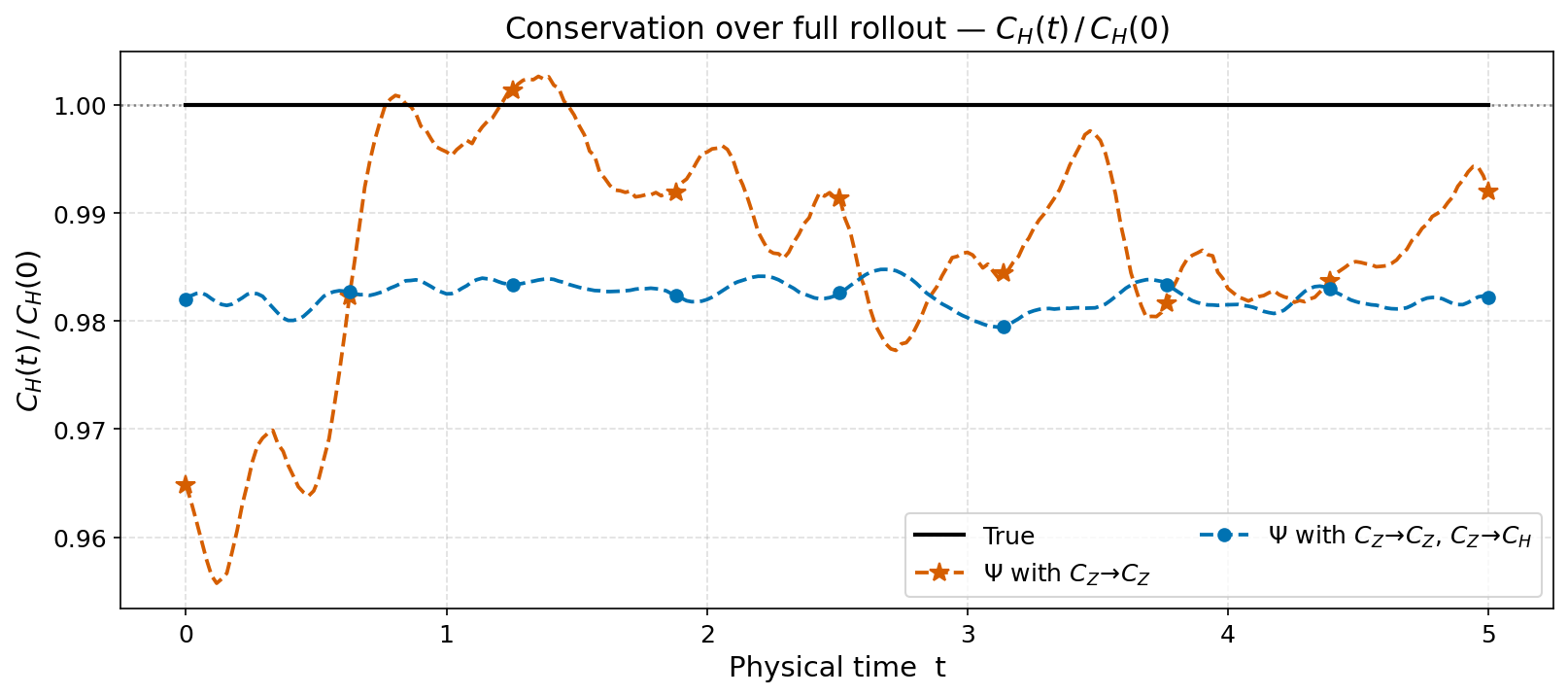}
    \caption{Ablation study for the wave equation. Normalized Hamiltonian
$C_H(t)/C_H(0)$ over the full horizon for two configurations that share
the hard skew-symmetric constraint and differ only in
whether pullback matching additionally aligns the latent invariant with the
decoder-induced Hamiltonian.}
    \label{fig:waveAblation}
\end{wrapfigure}
The ablation study in \Cref{fig:waveAblation} compares $C_H(t)/C_H(0)$
over the full horizon for two configurations. Both share the hard
skew-symmetric constraint that preserves $C_z$ exactly; they differ only
in whether pullback matching ($\lambda_{\mathrm{pullback}}=10^{-3}$) is
active. Without pullback alignment, exact preservation of $C_z$ provides
no direct control over $C_H$: the latent invariant and the physical
Hamiltonian remain unrelated representations. Adding pullback matching
aligns $C_z$ with the decoder-induced Hamiltonian, giving the closest
$C_H$ tracking.

\Cref{fig:waveRolloutQ,fig:waveRolloutP} show trajectory predictions for the
pullback-regularized model. The predicted fields closely match the reference
trajectories over the full horizon $t\in[0,5]$, including the persistence and
transport of fine-scale oscillatory structures. The absolute errors remain small
relative to the field amplitudes and are concentrated mainly near sharper spatial
features and phase-sensitive regions. This is a more demanding setting than the
heat equation, since errors are not naturally damped by the underlying dynamics.

The diagnostics in \Cref{fig:waveDiag} further confirm the effect of pullback
regularization. The predicted Hamiltonian remains close to the true value over
the full trajectory, while the latent quadratic invariant is conserved to
numerical precision. The model therefore captures both the observed wave fields
and the nondissipative structure of the underlying PDE.

\begin{figure}
    \centering
    \includegraphics[width=\linewidth]{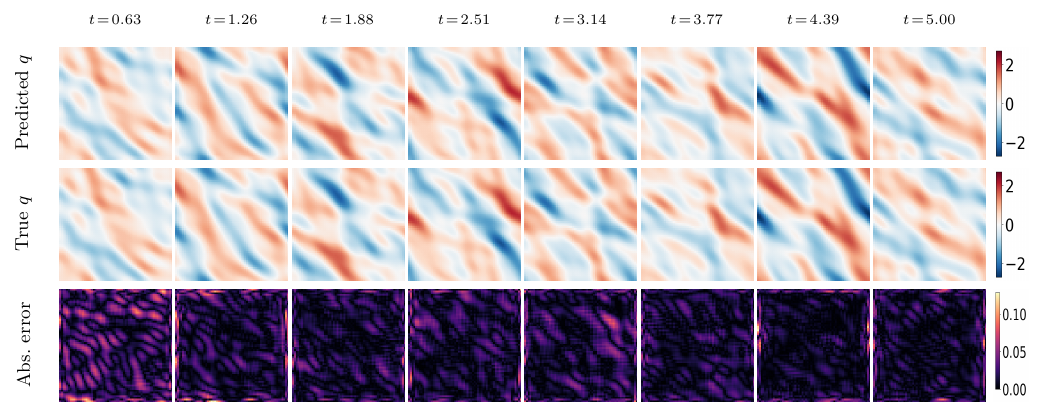}
    \caption{Wave equation trajectory prediction for the displacement field $q$. Predicted fields (row 1), true fields (row 2), and absolute errors (row 3) are shown at eight time snapshots over $t\in[0,5]$. The physics-conforming Latent Twins captures the oscillatory wave dynamics with small localized errors.}
    \label{fig:waveRolloutQ}
\end{figure}

\begin{figure}
    \centering
    \includegraphics[width=\linewidth]{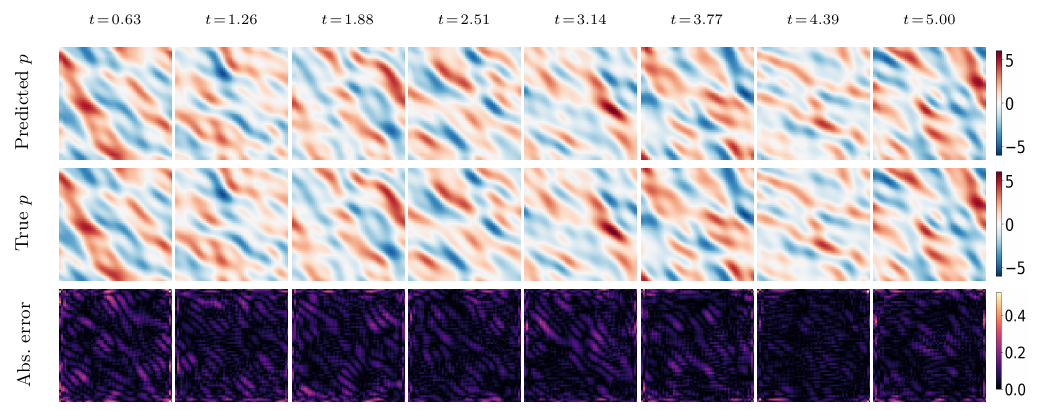}
    \caption{Wave equation trajectory prediction for the velocity field $p$. Predicted fields (row 1), true fields (row 2), and absolute errors (row 3) are shown at eight time snapshots over $t\in[0,5]$.}
    \label{fig:waveRolloutP}
\end{figure}

\begin{figure}
    \centering
    \includegraphics[width=0.85\linewidth]{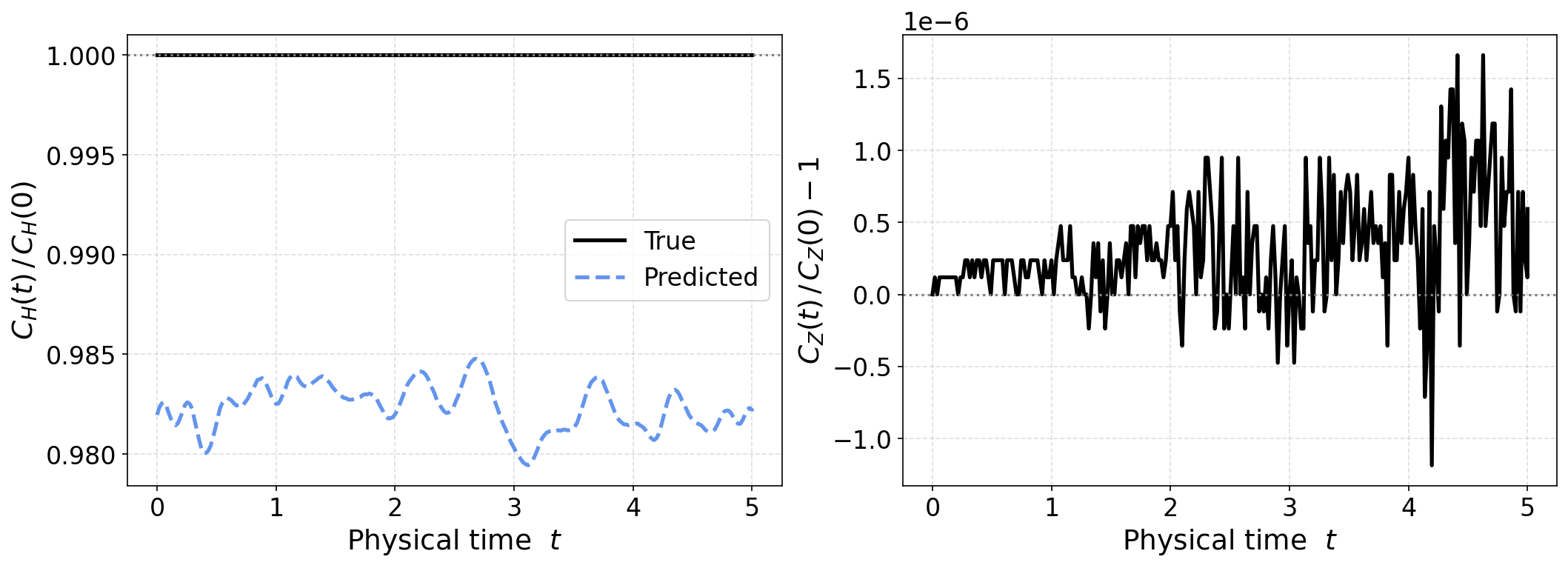}
    \caption{Conservation diagnostics for the wave Latent Twin. Left: normalized Hamiltonian $C_H(t)/C_H(0)$ for the true trajectory and surrogate prediction, with the predicted value deviating by approximately $2\%$ from the reference over the full horizon. Right: relative latent norm variation $C_Z(t)/C_Z(0) - 1$, confirming near-exact conservation of the quadratic latent invariant induced by the skew-symmetric latent operator.}
    \label{fig:waveDiag}
\end{figure}

\section{Conclusions and outlook}\label{sec:conclusion}

We introduced Physics-conforming Latent Twins, a latent solution-operator framework for learning surrogate dynamics that respect selected physical structures by design. The central idea is to use the latent space not only as a compression space, but as the space in which conservation laws, invariants, admissibility conditions, and dissipative structures are represented and enforced. In this view, physical constraints are not merely diagnostics applied after prediction; they are organizing principles that shape the class of admissible latent evolutions and, through decoding, the full surrogate dynamics in the physical state space.

On the theoretical side, we developed a constraint-transfer perspective that relates structural functionals in the physical state space to compatible functionals in latent space. We first showed that standard Latent Twins inherit physical structure a posteriori whenever the surrogate trajectory is close to the exact trajectory. We then proved a sharper structure-preserving estimate showing that, when a compatible latent constraint is enforced directly by the latent flow, the resulting physical defect after decoding is controlled by the compatibility error and the autoencoder reconstruction error, rather than by the full trajectory approximation error. This provides a theoretical justification for enforcing structure at the level of the latent dynamics, and motivates the numerical experiments in which conservative, dissipative, and invariant-preserving structures are imposed directly on the latent flow.

The experiments demonstrate the proposed framework on representative ODE and PDE benchmarks, but the formulation is not restricted to these settings. Physics-conforming Latent Twins apply whenever time-indexed or sequential states can be encoded, propagated in latent space, and decoded back to a meaningful state representation, and the structural functional can be adapted to the application at hand. More generally, the framework is not limited to homogeneous conservation laws: source terms, forcing, reactions, or boundary exchange can be incorporated as balance-law contributions, for instance by representing the corresponding contribution through the decoder-induced pullback and enforcing the resulting relation in latent space. This includes classical dynamical systems, discretized PDEs, iterative optimization processes, data-assimilation pipelines, and more general operator-learning or surrogate-modeling tasks. The main practical role of physics-conforming latent dynamics is to organize the learned representation so that prediction accuracy and structural fidelity are pursued jointly. This addresses a central tension in scientific machine learning: reliable surrogates benefit from physical structure, but the useful form of that structure may only become apparent after choosing a representation in which the dynamics are simple enough to model. Physics-conforming Latent Twins approach this problem by learning the representation, the latent evolution, and the structural constraint together.

Several limitations remain. First, compatibility between a prescribed latent constraint $\mathcal C_Z$ and the decoder-induced pullback of a physical constraint $\mathcal C$ is not automatic. It depends on the learned encoder--decoder pair and can be difficult to guarantee for highly nonlinear decoders. Second, hard architectural constraints provide strong preservation guarantees, but may reduce expressiveness if the chosen latent constraint class is not well matched to the physical structure after decoding. Third, the present theory primarily addresses global structural functionals, such as total mass, energy, or Lyapunov-type quantities, and does not yet guarantee local conservation laws or pointwise admissibility after reconstruction. Finally, the numerical experiments are intended as controlled demonstrations of the proposed design principles rather than exhaustive benchmarks across all possible surrogate-modeling settings.

These limitations suggest several natural extensions. A first direction is to learn or adapt the latent constraint itself, rather than prescribing it from a fixed tractable class, so that $\mathcal C_Z$ better approximates the decoder-induced pullback of the physical quantity while remaining enforceable by the latent flow. A second direction is to extend the framework to parameterized operator-learning settings, where the latent solution operator depends on physical parameters, geometries, forcing terms, boundary conditions, or discretization levels. A third direction is to move beyond global invariants toward local conservation laws and spatially resolved balance relations, which are central in many continuum systems.

More broadly, the framework points toward scientific foundation models built around reusable latent physics. Physics-conforming Latent Twins can serve as a cornerstone for such models by learning latent evolution modules that encode transferable principles, such as conservation, dissipation, stability, and admissibility. This makes the latent space more than a compression space: it becomes a structured space for prediction, inference, and control across systems, parameters, geometries, and application domains.

\medskip
\noindent {\bf Acknowledgments:} The authors thank Akil Narayan and Anthony Gruber for helpful discussions on pullbacks to physical laws.

\printcredits

\bibliographystyle{cas-model2-names}

\bibliography{references}



\end{document}